\documentclass[11pt]{article}

\usepackage{amssymb,amsfonts,amsmath,amsthm,amscd,dsfont,mathrsfs}
\usepackage{graphicx,float,psfrag,epsfig}
\usepackage{wrapfig}
\usepackage{pgf,pgfarrows,pgfnodes}

\DeclareMathAlphabet{\mathpzc}{OT1}{pzc}{m}{it}

\newtheorem{propo}{Proposition}[section]
\newtheorem{lemma}[propo]{Lemma}
\newtheorem{coro}[propo]{Corollary}
\newtheorem{thm}[propo]{Theorem}

\def\hm{{\hat m}}
\def\hv{{\hat v}}
\def\hl{{\hat \ell}}
\def\hr{{\hat r}}

\def\hx{{\hat x}}

\def\tx{{\widetilde x}}
\def\ty{{\widetilde y}}

\def\bp{{\bar p}}
\def\Bp{{\bf p}}
\def\Bx{{\bf x}}

\def\By{{\bf y}}

\def\Bz{{\bf z}}

\def\BA{{\bf A}}
\def\BZ{{\bf Z}}

\def\cN{{\cal N}}

\def\cF{{\cal F}}

\def\hBx{{\hat {\bf  x}}}

\def\bBp{{\bf{\bar p}}}

\def\cT{{\cal T}}
\def\ctT{{\cal \widetilde{T}}}

\def\reals{{\mathbb R}}
\def\prob{{\mathbb P}}
\def\Z{\mathbb Z}
\def\E{\mathbb E}
\def\sign{{{\rm sign}}}

\def\ind{{\mathbb I}}
\def\ones{{\mathds 1}}

\def\Gauss{{\cal N}}

\def\taskdegree{\ell}
\def\workerdegree{r}
\def\ntask{m}
\def\nworker{n}
\def\batch{T}
\def\estimate{\hat t}
\def\task{t}
\def\quality{q}
\def\meanquality{\mu}

\def\conditionalerror{{\cal E}}
\def\crowddist{{\cal F}}
\def\bpdist{F}

\def\tsigma{\tilde{\sigma}}
\def\error{\varepsilon}
\def\bI{{\bf I}}
\def\bG{{\bf G}}

	\definecolor{blue0}{RGB}{30,140,240}
	\definecolor{blue1}{RGB}{24,100,200}
	\definecolor{blue2}{RGB}{0,178,238}
	\definecolor{blue3}{RGB}{16,78,139}
	\definecolor{blue4}{RGB}{0,205,205}

\begin{document}


\title{Budget-Optimal Task Allocation \\ for Reliable Crowdsourcing Systems} 
\author{
{David R. Karger\thanks{Computer Science and Artificial Intelligence Laboratory
 and Department of EECS at Massachusetts Institute of Technology. Email: karger@mit.edu}, 
Sewoong Oh\thanks{Department of Industrial and Enterprise Systems Engineering at University of Illinois at Urbana-Champaign. Email: swoh@illinois.edu}, 
and Devavrat Shah\thanks{Laboratory for Information 
and Decision Systems and Department of EECS at Massachusetts Institute of Technology. Email: 
devavrat@mit.edu. This work was supported in 
parts by NSF EMT project, AFOSR Complex Networks project and Army Research Office under MURI
Award 58153-MA-MUR.} }\\
}

\date{\today}

\maketitle

\begin{abstract}

Crowdsourcing systems, in which numerous tasks are electronically
distributed to numerous ``information piece-workers'', have emerged as
an effective paradigm for human-powered solving of large scale
problems in domains such as image classification, data entry, optical
character recognition, recommendation, and proofreading.  Because
these low-paid workers can be unreliable, nearly all such systems must
devise schemes to increase confidence in their answers, typically by
assigning each task multiple times and combining the answers in an appropriate manner, e.g. majority voting. 

In this paper, we consider a general model of such crowdsourcing tasks
and pose the problem of minimizing the total price (i.e., number of
task assignments) that must be paid to achieve a target overall
reliability.  We give a new algorithm for deciding which tasks to
assign to which workers and for inferring correct answers from the
workers' answers.  We show that our algorithm, 
inspired by belief propagation and low-rank
matrix approximation, significantly outperforms majority voting and, in
fact, is optimal through comparison to an oracle that
knows the reliability of every worker. 
Further, we compare our approach with a more general class of algorithms 
which can dynamically assign tasks. 
By adaptively deciding which questions to ask to the next arriving worker, 
one might hope to reduce uncertainty more efficiently. 
We show that, perhaps surprisingly, the minimum price necessary to  
achieve a target reliability scales in the same manner under both adaptive and non-adaptive scenarios. 
Hence, our non-adaptive approach is order-optimal under both scenarios.
This strongly relies on the fact that workers are fleeting and can not be exploited. 
Therefore, architecturally, our results suggest that building a reliable worker-reputation system 
is essential to fully harnessing the potential of adaptive designs.

\end{abstract}

\newpage

%
%
\section{Introduction}
\label{sec:introduction}

\noindent {\bf Background.} 
Crowdsourcing systems have emerged as an
effective paradigm for human-powered problem solving and are now in
widespread use for large-scale data-processing tasks such as image
classification, video annotation, form data entry, optical character
recognition, translation, recommendation, and proofreading.
Crowdsourcing systems such as Amazon Mechanical
Turk\footnote{http://www.mturk.com}, 
establish a market where a ``taskmaster'' can submit batches of small tasks to be
completed for a small fee by any worker choosing to pick them up.  For
example a worker may be able to earn a few cents by indicating which
images from a set of 30 are suitable for children (one of the benefits
of crowdsourcing is its applicability to such highly subjective questions).

Because these crowdsourced tasks are tedious and the pay is low, errors are common
even among workers who make an effort. 
At the extreme, some workers are ``spammers'', submitting arbitrary answers independent of the
question in order to collect their fee.  Thus, all crowdsourcers need
strategies to ensure the reliability of their answers.  
When the system allows the crowdsourcers to identify and reuse particular workers, 
a common approach is to manage a pool of reliable workers in an explore/exploit fashion. 
However in many crowdsourcing platforms such as Amazon Mechanical Turk, 
the worker crowd is large, anonymous, and transient, and 
it is generally difficult to build up a trust relationship with particular workers.\footnote{For
  certain high-value tasks, crowdsourcers can use entrance exams to
  ``prequalify'' workers and block spammers, but this increases the
  cost of the task and still provides no guarantee that the workers
  will try hard after qualification.}  It is also difficult to
condition payment on correct answers, as the correct answer may never
truly be known and delaying payment can annoy workers and make it
harder to recruit them to your task next time.  Instead, most crowdsourcers
resort to redundancy, giving each task to multiple workers, paying
them all irrespective of their answers, 
and aggregating the results by some method such
as majority voting. 

For such systems there is a natural core optimization problem to be
solved.  Assuming the taskmaster wishes to achieve a certain
reliability in her answers, how can she do so at minimum cost
(which is equivalent to asking how she can do so while asking the
fewest possible questions)?  

Several characteristics of crowdsourcing systems make this problem
interesting.  Workers are neither persistent nor identifiable; each
batch of tasks will be solved by a worker who may be completely new
and who you may never see again.  Thus one cannot identify and reuse
particularly reliable workers.  
Nonetheless, by comparing one worker's answer to others' on the same
question, it is possible to draw conclusions about a worker's
reliability, which can be used to weight their answers to other
questions in their batch. 
However, batches must be of manageable size,
obeying limits on the number of tasks that can be given to a single worker.

Another interesting aspect of this problem is the \emph{choice of task assignments}. 
Unlike many inference tasks which makes inferences 
based on a fixed set of signals, our algorithm can choose which
signals to measure by deciding which questions to include in which batches.
In addition, there are several plausible options: for example, we might
choose to ask a few ``pilot questions'' to each worker (just like
a qualifying exam) to decide on the reliability of the worker. 
Another possibility is to first ask few questions and based on the answers decide
to ask more questions or not.  We would like to understand the role of
all such variations in the overall optimization of budget for reliable task
processing. 

In the remainder of this section, we will define a formal probabilistic model
that captures these aspects of the problem. 
We consider both a \emph{non-adaptive scenario}, in which 
all questions are asked simultaneously and 
all the responses are collected simultaneously, 
and an {\em adaptive scenario}, 
in which one may adaptively choose which tasks to assign to the next arriving worker 
based on all the previous answers collected thus far. 
We provide a non-adaptive task allocation scheme and 
an inference algorithm based on low-rank matrix approximations and belief propagation. 
We will then show that 
our algorithm is order-optimal: 
for a given target error rate, it spends only a constant factor times 
the minimum necessary to achieve that error rate. 
The optimality is established through comparisons to 
the best possible {\em non-adaptive} task allocation scheme and  
an oracle estimator that can make optimal decisions based on extra information provided by an oracle. 
In particular, we derive a parameter $\quality$ that characterizes the `collective' reliability of the crowd, 
and show that to achieve target reliability $\error$, it is both 
necessary and sufficient to replicate each task $\Theta(1/\quality\log(1/\error))$ times.
This leads to the next question of interest: 
by using adaptive task assignment, can we ask fewer questions and still achieve the same error rate? 
We, somewhat surprisingly, show that the optimal costs under this 
adaptive scenario scale in the same manner as the non-adaptive scenario; 
asking questions adaptively does not help! 

\vspace{0.2cm} 
\noindent {\bf Setup.} We consider the following probabilistic model for crowdsourcing. 
There is a set of $\ntask$ binary tasks 
which is associated with unobserved `correct' solutions: $\{\task_i\}_{i\in[\ntask]}\in\{\pm1\}^\ntask$. 
Here and after, we use $[N]$ to denote the set of first $N$ integers. 
In the image categorization example stated earlier, a set of tasks corresponds to 
labeling $\ntask$ images as suitable for children $(+1)$ or not $(-1)$.  
We will be interested in finding the true solutions by querying
noisy workers who arrive one at a time in an on-line fashion. 

An algorithmic solution to crowdsourcing consists of two components: 
a task allocation scheme and an inference algorithm. 
At {\em task allocation phase} queries are made sequentially according to the following rule. 
At $j$-th step, the task assignment scheme chooses a subset $\batch_j\subseteq[m]$ 
of tasks to be assigned to the next arriving noisy worker. 
{ The only constraint on the choice of the batch is that 
the size $|\batch_j|$ must obey some limit on the number of tasks that can be given to a single worker.
Let $\workerdegree$ denote such a limit on the number of tasks that can be assigned to a single worker, 
such that all batches must satisfy $|T_j|\leq\workerdegree$.}
Then, a worker $j$ arrives, whose latent reliability is parametrized by $p_j\in[0,1]$. 
For each assigned task, this worker gives a noisy answer such that
\begin{eqnarray*}
  \label{eq:defA}
  \BA_{ij} &=& \left\{ 
      \begin{array}{rl} \task_i &\text{ with probability }p_j \;,\\
      -\task_i &\text{ otherwise}\;, \end{array}
	      \right.
\end{eqnarray*}
and $\BA_{ij}=0$ if $i\notin \batch_j$. 
(Throughout this paper, we use boldface characters 
to denote random variables and random matrices 
unless it is clear from the context.) 
The next assignment $\batch_{j+1}$ can be chosen adaptively,
taking into account all of the previous assignments and the answers collected thus far. 
This process is repeated until the task assignment scheme decides to stop, 
typically when the total number of queries meet a certain budget constraint. 
Then, in the subsequent {\em inference phase}, an inference algorithm makes a final estimation of the true answers.

We say a task allocation scheme is {\em adaptive} if the choice of $\batch_j$ depends on the 
answers collected on previous steps, and it is {\em non-adaptive} if it does not depend on the answers. 
In practice, one might prefer using a non-adaptive scheme, since 
assigning all the batches simultaneously and having all 
the batches of tasks processed in parallel reduces latency. 
However, by switching to an adaptive task allocation, one might be able to 
reduce uncertainty more efficiently. 
We investigate this possibility in Section~\ref{sec:minimax_adaptive}, and 
show that the gain from adaptation is limited. 

Note here that at the time of assigning tasks $\batch_j$ for a next arriving worker $j$, 
the algorithm is not aware of the latent reliability of the worker. 
This is consistent with how real-world crowdsourcing works, 
since taskmasters typically have no choice over which worker is going to pick up which batch of tasks.
Further, we make the pessimistic assumption that 
workers are neither persistent nor identifiable; each batch of tasks $\batch_j$ will be solved by a
worker who may be completely new and who you may never see again. 
Thus one cannot identify and
reuse particularly reliable workers. This is a different setting from adaptive games \cite{LW89}, 
where you have a sequence of trials and a set of predictions is made at each step by a pool of experts. 
In adaptive games, you can identify reliable experts from their past performance 
using techniques like multiplicative weights, 
whereas in crowdsourcing you cannot hope to exploit any particular worker.

The latent variable $p_j$ captures how some workers are more diligent or have more expertise than others, 
while some other workers might be trying to cheat. 
The random variable $\BA_{ij}$ is independent of any other event given $p_j$.  
The underlying assumption here is that 
the error probability of a worker does not depend on the particular task and 
all the tasks share an equal level of difficulty. 
Hence, each worker's performance is consistent across different tasks. 
We discuss a possible generalization of this model in Section~\ref{sec:discuss}.

We further assume that the reliability of workers 
$\{\Bp_j\}$ are independent 
and identically distributed random variables with a given distribution on $[0,1]$. 
As one example we define the {\em spammer-hammer model}, where  
each worker is either a `hammer' with probability $\quality$ or is a `spammer' with probability $1-\quality$. 
A hammer answers all questions correctly, meaning $\Bp_j=1$,  
and a spammer gives random answers, meaning $\Bp_j=1/2$. 
It should be noted that the meaning of a `spammer' might be different from 
its use in other literature. 
In this model, a spammer is a worker
who gives uniformly random labels independent of the true label. 
In other literature in crowdsourcing, 
the word `spammer' has been used, for instance, 
to refer to a worker who always gives `$+$' labels \cite{RY12}.
Another example is the beta distribution with some parameters $\alpha>0$ and $\beta>0$ 
($f(p)=p^{\alpha-1}(1-p)^{\beta-1}/B(\alpha,\beta)$ for a proper normalization $B(\alpha,\beta)$) \cite{Holmes,RYZ10}.
A distribution of $\Bp_j$ characterizes a crowd, 
and the following parameter plays an important role in
capturing the `collective quality' of this crowd, as will be clear from our main results:
\begin{eqnarray*}
	\quality &\equiv& \E[(2\Bp_j-1)^2] \;.
\end{eqnarray*}
A value of $\quality$ close to one indicates that a large proportion of the workers are diligent, 
whereas $\quality$ close to zero indicates that there are many spammers in the crowd. 
The definition of $\quality$ is consistent with use of $\quality$ in the spammer-hammer model   
and in the case of beta distribution, $\quality = 1-(4\alpha\beta/((\alpha+\beta)(\alpha+\beta+1)))$.
We will see later that our bound on the achievable error rate 
depends on the distribution only through this parameter $\quality$. 

When the crowd population is large enough such that we do not need to distinguish 
whether the workers are `sampled' with or without replacement, then 
it is quite realistic to assume the existence of a prior distribution for $\Bp_j$. 
In particular, it is met if we simply
randomize the order in which we upload our task batches, since this
will have the effect of randomizing which workers perform which
batches, yielding a distribution that meets our requirements. 
The model is therefore quite general. 
On the other hand, it is not realistic to assume that we know what the prior is. 
To execute our inference algorithm for a given number of iterations, 
we do not require any knowledge of the distribution of the reliability. 
However, $\quality$ is necessary in order to determine 
how many times a task should be replicated 
and how many iterations we need to run to achieve a certain target reliability.  
We discuss a simple way to overcome this limitation in Section \ref{sec:theory}.

The only assumption we make about the distribution is that there is a bias towards the right answer, i.e. 
$\E[\Bp_j]>1/2$. Without this assumption, 
we can have a `perfect' crowd with $\quality = 1$, but everyone is adversarial, $\Bp_j = 0$.
Then, there is no way we can correct for this. 
Another way to justify this assumption is to define
the ``ground truth'' of the tasks as what the majority of the crowd agrees on. 
We want to learn this consensus efficiently without having to query everyone in the crowd for every task.
{ If we use this definition of the ground truth, 
then it naturally follows that the workers are on average more likely to be correct.} 

Throughout this paper, we are going to assume that 
there is a fixed cost you need to pay for each response you get regardless of the quality of the response, 
such that the total cost is proportional to the total number of queries. 
When we have a given target accuracy we want to achieve, and 
under the probabilistic crowdsourcing model described in this section, 
we want to design a task allocation scheme 
and an inference algorithm that can 
achieve this target accuracy with minimal cost. 

%
%
\vspace{0.2cm}
\noindent {\bf Possible deviations from our model.} 
Some of the main assumptions we make on how crowdsourcing systems work are 
$(a)$ workers are neither identifiable nor reusable, 
$(b)$ every worker is paid the same amount regardless of their performance,  
and $(c)$ each worker completes only one batch and she completes all the tasks in that batch. 
In this section, we discuss common strategies used 
in real crowdsourcing that might deviate from these assumptions.

First, there has been growing interest recently 
in developing algorithms to efficiently identify good workers assuming that 
worker identities are known and {\em workers are reusable}. 
Imagine a crowdsourcing platform where there are a fixed pool of identifiable workers 
and we can assign the tasks to whichever worker  we choose to. 
In this setting, adaptive schemes can be used to
significantly improve the accuracy while minimizing the total number of queries. It is natural to
expect that by first exploring to find better workers and then exploiting them in the following rounds,
one might be able to improve performance significantly. 
Donmez et al. \cite{DCS09} proposed IEThresh
which simultaneously estimates worker accuracy and actively selects a subset of workers with high
accuracy. Zheng et al. \cite{ZSD10} proposed a two-phase approach to identify good workers in the first
phase and utilize the best subset of workers in the second phase. Ertekin et al. \cite{EHR11} proposed using
a weighted majority voting to better estimate the true labels in CrowdSense, which is then used to
identify good workers. 

The power of such exploration/exploitation approaches were demonstrated on numerical experiments, 
however none of these approaches are tested on real-world crowdsourcing.  
All the experiments are done using {\em pre-collected} datasets. 
Given these datasets they simulate a labor market 
where they can track and reuse any workers they choose to. 
The reason that the experiments are done on such simulated labor markets, 
instead of on popular crowdsourcing platforms such as Amazon Mechanical Turk, is that 
on real-world crowdsourcing platforms it is almost impossible to track workers. 
Many of the popular crowdsourcing platforms are completely open labor markets 
where the worker crowd is large and transient. 
Further, oftentimes it is the workers who choose which tasks they want to work on, 
hence the taskmaster cannot reuse particular workers. 
For these reasons, we assume in this paper that the workers are fleeting 
and provide an algorithmic solution that works even when workers are not reusable. 
We show that any taskmaster who wishes to outperform our algorithm must adopt 
complex worker-tracking techniques. 
Furthermore, no worker-tracking technique has been developed that 
has been proven to be foolproof. 
In particular, it is impossible to prevent a worker from starting over with a new account. 
Many tracking algorithms are susceptible to this attack.

Another important and closely related question that 
has not been formally addressed in crowdsourcing literature is 
how to {\em differentiate the payment} based on the inferred accuracy in order to incentivize good workers.
Regardless of whether the workers are identifiable or not, when all the tasks are completed we get an 
estimate of the quality of the workers. 
It would be desirable to pay the good workers more in order to incentivize them to work for us in the 
future tasks. 
For example, bonuses are built into Amazon Mechanical Turk 
to be granted at the taskmaster's discretion, 
but it has not been studied how to use bonuses optimally.  
This could be an interesting direction for future research.
  
It has been observed that 
increasing the cost on crowdsourcing platforms does not 
directly lead to higher quality of the responses \cite{MW09}. 
Instead, increasing the cost only leads to faster responses. 
Mason and Watts \cite{MW09} attributes this counterintuitive findings to an ``anchoring'' effect. 
When the (expected) payment is higher, workers perceive the value of their work to be greater as well. 
Hence, they are no more motivated than workers who are paid less. 
However, these studies were done in isolated experiments, 
and the long term effect of taskmasters' keeping a good reputation still needs to be understood. 
Workers of Mechanical Turk can manage reputation of the taskmasters using for instance 
Turkopticon\footnote{http://turkopticon.differenceengines.com}, 
a Firefox extension that allows you to rate taskmasters and 
view ratings from other workers.
Another example is Turkernation\footnote{http://turkernation.com}, 
an on-line forum where workers and taskmasters can discuss Mechanical Turk and leave feedback.

Finally, in Mechanical Turk, it is typically the workers who {\em choose 
which tasks they want to work on} and when they want to stop. 
Without any regulations, they might respond to multiple batches of your tasks 
or stop in the middle of a batch. 
It is possible to systematically prevent the same worker from coming back and 
repeating more than one batch of your tasks. 
For example, on Amazon's Mechanical Turk, 
a worker cannot repeat the same task more than once. 
However, it is difficult to guarantee that a worker completes all the tasks in a batch she started on. 
In practice, there are simple ways to ensure this by, for instance, conditioning 
the payment on completing all the tasks in a batch.

A problem with restricting the number of tasks assigned to each worker (as we propose in Section~\ref{sec:algorithm}) 
is that it might take a long time to have all the batches completed. 
Letting the workers choose how many tasks they want to complete 
allows a few eager workers to complete enormous amount of tasks. 
However, if we restrict the number of tasks assigned to each worker, 
we might need to recruit more workers to complete all the tasks. 
This problem of tasks taking long time to finish is not 
just restricted to our model, but is a very common problem 
in open crowdsourcing platforms. 
Ipeirotis \cite{Ipe10} studied the completion time of tasks on Mechanical Turk 
and observed that it follows a heavy tail distribution according to a power law.
Hence, for some tasks it takes significant amount of time to finish. 
A number of strategies have been proposed to complete tasks on time. 
This includes optimizing pricing policy \cite{FHI11}, 
continuously posting tasks to stay on the first page \cite{BJJ10,CHMA10}, 
and having a large amount of tasks available \cite{CHMA10}. 
These strategies are effective in attracting more workers fast. 
However, in our model, we assume there is no restrictions on the latency and 
we can wait until all the batches are completed, and if we have good strategies to reduce 
worker response time, such strategies could be incorporated  into our system design. 

%
%
\vspace{0.2cm}
\noindent {\bf Prior work.} 
Previous crowdsourcing system designs have focused on 
developing inference algorithms assuming that the task assignments are fixed and the workers' responses are already given.
None of the prior work on crowdsourcing provides any systematic treatment of task assignment  
under the crowdsourcing model considered in this paper. 
To the best of our knowledge, we are the first to study both aspects of crowdsourcing together 
and, more importantly, establish optimality. 

A naive approach to solve the inference problem, 
which is widely used in practice, is majority voting. 
Majority voting simply follows what the majority of workers agree on. 
When we have many spammers in the crowd,  majority voting is error-prone 
since it gives the same weight to all the responses, 
regardless of whether they are from a spammer or a diligent workers. 
We will show in Section \ref{sec:minimax_nonadaptive} that 
majority voting is provably sub-optimal 
and can be significantly improved upon.

If we know how reliable each worker is, then 
it is straightforward to find the maximum likelihood estimates: 
compute the weighted sum of the responses weighted by the log-likelihood. 
Although, in reality, we do not have this information,  
it is possible to learn about a worker's reliability 
by comparing one worker's answer to others'. 
This idea was first proposed by Dawid and Skene, 
who introduced an iterative algorithm based on expectation maximization (EM) \cite{DS79}.  
{ They considered the problem of classifying patients based on 
labels obtained from multiple clinicians. 
They introduce a simple probabilistic model describing the clinicians' responses, 
and gave an algorithmic solution based on EM. 
This model, which is described in Section~\ref{sec:discuss}, 
is commonly used in modern crowdsourcing settings to explain how workers make mistakes in classification tasks \cite{SPI08}.} 

This heuristic algorithm iterates the following two steps. 
In the M-step, 
the algorithm estimates the error probabilities of the workers 
that maximizes the likelihood 
using the current estimates of the answers.
In the E-step, the algorithm estimates the likelihood of the answers 
using the current estimates of the error probabilities. 
More recently, a number of algorithms followed this EM approach 
based on a variety of probabilistic models \cite{smyth95,whitehill09,Ray10}. 
The crowdsourcing model we consider in this paper is a special case of these models, 
and we discuss their relationship in Section \ref{sec:discuss}. 
The EM approach has also been widely applied in classification problems,  
where a set of labels from low-cost noisy workers is used to find a good classifier \cite{JG03,Ray10}.
Given a fixed budget, there is a trade-off between 
acquiring a larger training dataset or acquiring a smaller dataset but with more labels per data point. 
Through extensive experiments, Sheng, Provost and Ipeirotis \cite{SPI08} 
show that getting repeated labeling can give considerable advantage.
 
Despite the popularity of the EM algorithms, 
the performance of these approaches are only empirically evaluated  
and there is no analysis that gives performance guarantees. 
In particular, EM algorithms 
are highly sensitive to the initialization used,  
making it difficult to predict the quality of the resulting estimate. 
Further, the role of the task assignment is not at all understood with the EM algorithm 
(or for that matter any other algorithm). 
We want to address both questions of task allocation and inference together, 
and devise an algorithmic solution that can achieve 
minimum error from a fixed budget on the total number of queries. 
When we have a given target accuracy, 
such an algorithm will achieve this target accuracy 
with minimum cost. 
Further, we want to provide a strong performance guarantee for this approach and 
show that it is close to a fundamental limit on what the best algorithm can achieve. 

%
%

\vspace{0.2cm}
\noindent {\bf Contributions.} 
In this work, we provide the first rigorous treatment of 
both aspects of designing a reliable crowdsourcing system: 
task allocation and inference.  
We provide both an order-optimal task allocation scheme (based on random graphs) 
and an order-optimal algorithm for inference (based on low-rank approximation and belief propagation) 
on that task assignment. 
We show that our algorithm, which is non-adaptive, performs as well (for the worst-case worker distribution) 
as the optimal oracle estimator which can use any adaptive task allocation scheme. 

Concretely, given a target probability of error $\error$ and a crowd with collective quality $\quality$,  
we show that spending a budget which scales as $O(\,(1/\quality)\log(1/\error)\,)$ 
is sufficient to achieve probability of error less than $\error$ using our approach. 
We give a task allocation scheme and an inference algorithm with 
runtime which is linear in the total number of queries (up to a logarithmic factor).  
Conversely, we also show that 
using the best adaptive task allocation scheme 
together with the best inference algorithm, 
and under the worst-case worker distribution,  
this scaling of the budget 
in terms of $\quality$ and $\error$ is unavoidable. 
No algorithm can achieve error less than $\error$ with number of queries smaller than 
$(C/\quality)\log(1/\error)$ with some positive universal constant $C$.  
This establishes that our algorithm is worst-case optimal up to a constant factor in the required budget. 

Our main results show that our non-adaptive algorithm is worst-case optimal 
and there is no significant gain in using an adaptive strategy. 
We attribute this limit of adaptation to the fact that, in existing platforms such as Amazon's Mechanical Turk, 
the workers are fleeting and the system does not allow for exploiting good workers. 
Therefore, a positive message of this result 
is that a good ‘rating system’ for workers is essential to truly benefit from crowdsourcing platforms using adaptivity.

Another novel contribution of our work is the analysis technique. 
The iterative inference algorithm we introduce operates on real-valued messages 
whose distribution is a priori difficult to analyze.
To overcome this challenge, we develop a novel technique of establishing 
that these messages are sub-Gaussian 
and compute the parameters recursively in a closed form. 
This allows us to prove the sharp result on the error rate.  
This technique could be of independent interest in 
analyzing a more general class of message-passing algorithms.

%
%
\section{Main result}
\label{sec:result}

To achieve a certain reliability in our answers with minimum number of queries, 
we propose using random regular graphs for task allocation 
and introduce a novel iterative algorithm to infer the correct answers.
While our approach is {\em non-adaptive}, we show that it is sufficient to achieve an order-optimal performance
when compared to the best possible approach using {\em adaptive} task allocations. 
Precisely, we prove an upper bound on the resulting error when using our approach 
and a matching lower bound on the minimax error rate achieved by 
the best possible {adaptive} task allocation together with an optimal inference algorithm. 
This shows that our approach is minimax optimal up to a constant factor:  
it requires only a constant factor times the minimum necessary budget 
to achieve a target error rate under the worst-case worker distribution. 
We then present the intuitions behind our inference algorithm through connections to 
low-rank matrix approximations and belief propagation.

%
%
\subsection{Algorithm}
\label{sec:algorithm}

%
%
\vspace{0.2cm}
\noindent {\bf Task allocation.} 
We use a non-adaptive scheme which makes all the task assignments before any worker arrives. 
This amounts to designing a bipartite graph with 
one type of nodes corresponding to each of the tasks and 
another set of nodes corresponding to each of the batches. 
An edge $(i,j)$ indicates that task $i$ is included in batch $\batch_j$. 
Once all $\batch_j$'s are determined according to the graph, 
these batches are submitted simultaneously to the crowdsourcing platform. 
Each arriving worker will pick up one of the batches and complete all the tasks in that batch. 
We denote by $j$ the worker working on $j$-th batch $\batch_j$. 

To design a bipartite graph, the taskmaster first 
makes a choice of how many workers to assign to each task 
and how many tasks to assign to each worker. 
The task degree $\taskdegree$ is typically determined by how much resources (e.g. money, time, etc.) one can spend on the tasks. 
The worker degree $\workerdegree$ is typically determined by how many tasks are manageable for a worker depending on the application. 
The total number of workers that we need is 
automatically determined as $\nworker=\ntask\taskdegree/\workerdegree$, 
since the total number of edges has to be consistent. 

We will show that with such a regular graph, 
you can achieve probability of error which is quite close to 
a lower bound on what any inference algorithm can achieve with any task assignment. 
In particular, this includes all possible graphs which might have irregular degrees 
or have very large worker degrees (and small number of workers)  
conditioned on the total number of edges being the same. 
This suggests that, among other things, there is no significant gain in using an irregular graph. 

We assume that the total cost that must be paid is proportional to the total number of edges  
and not the number of workers. If we have more budget we can increase $\taskdegree$.  
It is then natural to expect the probability of error to decrease, since we are collecting more responses.
We will show that the error rate decreases exponentially in $\taskdegree$ as $\taskdegree$ grows. 
However, increasing $\workerdegree$ does not incur increase in the cost and 
it is not immediately clear how it affects the performance. 
We will show that 
with larger $\workerdegree$ we can learn more about the workers and 
the error rate decreases as $\workerdegree$ increases.  
However, how much we can gain by increasing the worker degree is limited.

Given the task and worker degrees, 
there are multiple ways to generate a regular bipartite graph. 
We want to choose a graph that will minimize the probability of error. 
Deviating slightly from regular degrees, we propose using a simple random construction 
known as {\em configuration model} in random graph literature \cite{RU08,Bol01}. 
We start with $[\ntask]\times[\taskdegree]$ half-edges for task nodes and 
$[\nworker]\times[\workerdegree]$ half-edges for the worker nodes, 
and pair all the half-edges according to a random permutation of $[\ntask\taskdegree]$. 
The resulting graph might have 
multi-edges where two nodes are connected by more than one edges. 
However, they are very few in thus generated random graph as long as $\taskdegree\ll \nworker$, 
whence we also have $\workerdegree\ll \ntask$. 
Precisely, the number of double-edges 
in the graph converges in distribution to Poisson distribution 
with mean $(\taskdegree-1)(\workerdegree-1)/2$ \cite[Page 59 Exercise 2.12]{Bol01}. 
The only property that we need for the main result to hold is that
the resulting random graph converges locally to a random tree in probability in the large system limit. 
This enables us to analyze the performance of our inference algorithm 
and provide sharp bounds on the probability of error.  


The intuition behind why random graphs are good for our inference problem 
is related to the spectral gap of random matrices. 
In the following, we will use the (approximate) top singular vector of a weighted adjacency matrix of the random graph to 
find the correct labels. 
Since, sparse random graphs are excellent expanders with large spectral gaps,   
this enables us to reliably separate the low-rank structure 
from the data matrix which is perturbed by random noise. 

%
%
\vspace{0.2cm}
\noindent {\bf Inference algorithm.} 
We are given a task allocation graph $G\big([\ntask]\cup[\nworker],E\big)$ 
where we connect an edge $(i,j)$ if a task $i$ is assigned to a worker $j$. 
In the following, we will use indexes $i$ for a $i$-th task node and $j$ for a $j$-th worker node. 
We use $\partial i$ to denote the neighborhood of node $i$.  
Each edge $(i,j)$ on the graph $G$ has a corresponding worker response $A_{ij}$. 

To find the correct labels from the given responses of the workers, 
we introduce a novel iterative algorithm. 
This algorithm is inspired by the celebrated belief propagation algorithm 
and low-rank matrix approximations. 
The connections are explained in detail in Section~\ref{sec:relationlowrank} and \ref{sec:relationbp}, 
along with mathematical justifications. 

The algorithm operates on real-valued task messages 
$\{x_{i\to j}\}_{(i,j)\in E}$ and worker messages $\{y_{j\to i}\}_{(i,j)\in E}$. 
A task message $x_{i\to j}$ represents the log-likelihood of task $i$ being a positive task, 
and a worker message $y_{j\to i}$ represents how reliable worker $j$ is. 
We start with the worker messages initialized as independent Gaussian random variables, 
although the algorithm is not sensitive to a specific initialization 
as long as it has a strictly positive mean. 
We could also initialize all the messages to one, 
but then we need to add extra steps in the analysis to ensure that this is not a degenerate case.
At $k$-th iteration, the messages are updated according to 
\begin{eqnarray}
	\label{eq:messageupdate1}
	x^{(k)}_{i\to j} &=& \sum_{j'\in\partial i \setminus j} A_{ij'}y^{(k-1)}_{j'\to i}\;,\;\; \text{ for all $(i,j)\in E$\;, and}\\ 
	\label{eq:messageupdate2}
	y^{(k)}_{j\to i} &=& \sum_{i'\in\partial j \setminus i} A_{i'j}x^{(k)}_{i'\to j}\;,\;\; \text{ for all $(i,j)\in E$}\;,
\end{eqnarray}
where $\partial i$ is the neighborhood of the task node $i$ and 
$\partial j$ is the neighborhood of the worker node $j$. 
At task update, we are giving more weight to the answers that came from more trustworthy workers. 
At worker update, we increase our confidence in that worker if 
the answers she gave on another task, $A_{i'j}$, has the same sign as what we believe, $x_{i'\to j}$.
Intuitively, a worker message represents our belief on how `reliable' the worker is.   
Hence, our final estimate is a weighted sum 
of the answers weighted by each worker's reliability: 
\begin{eqnarray*}
	\estimate_i^{(k)} \;= \;\sign\Big(\sum_{j\in\partial i}A_{ij}y_{j\to i}^{(k-1)}\Big)\;.
\end{eqnarray*}

\begin{center}
\begin{tabular}{ll}
\hline
\vspace{-.35cm}\\
\multicolumn{2}{l}{ Iterative Algorithm}\\
\hline
\vspace{-.35cm}\\
\multicolumn{2}{l}{{\bf Input:} $E$, $\{A_{ij}\}_{(i,j)\in E}$, $k_{\rm max}$} \\
\multicolumn{2}{l}{{\bf Output:} Estimate $\estimate\in\{\pm1\}^\ntask$}\\
1:  & {\bf For all} $(i,j)\in E$ {\bf do}\\
   & \hspace{0.6cm} Initialize $y^{(0)}_{j\to i}$ with random $Z_{ij}\sim\cN(1,1)$ ; \\
2:  & {\bf For} $k=1,\ldots,k_{\rm max}$ {\bf do}\\
   & \hspace{0.6cm} {\bf For all} $(i,j)\in E$ {\bf do}
     \hspace{0.3cm}$x^{(k)}_{i\to j} \leftarrow \sum_{j'\in\partial i \setminus j} A_{ij'}y^{(k-1)}_{j'\to i}$ ; \\
   & \hspace{0.6cm} {\bf For all} $(i,j)\in E$ {\bf do}
     \hspace{0.3cm}$y^{(k)}_{j\to i} \leftarrow \sum_{i'\in\partial j \setminus i} A_{i'j}x^{(k)}_{i'\to j}$ ;\\
3:  & {\bf For all} $i\in [m]$ {\bf do}
     \hspace{0.3cm}$x_i \leftarrow \sum_{j\in\partial i} A_{ij}y^{(k_{\rm max}-1)}_{j\to i}$ ;\\
4:  & Output estimate vector $\estimate^{(k)}=\{\sign(x_i)\}$ .\\
\hline
\end{tabular}
\end{center}

While our algorithm is inspired by the standard Belief Propagation (BP) algorithm 
for approximating max-marginals \cite{Pearl88,YFW03}, our algorithm is original and 
overcomes a few limitations of the standard BP for this inference problem under the crowdsourcing model. 
First, the iterative algorithm does not require any knowledge of the prior distribution of $\Bp_j$, 
whereas the standard BP requires it as explained in detail in Section~\ref{sec:relationbp}. 
Second, the iterative algorithm is provably {order}-optimal for this crowdsourcing problem. 
We use a standard technique, known as {\em density evolution}, to analyze the performance of our message-passing algorithm. 
Although we can write down the density evolution equations for the standard BP for crowdsourcing, 
it is not trivial to describe or compute the densities, analytically or numerically. 
It is also very simple to write down the density evolution equations (cf. \eqref{eq:de1} and \eqref{eq:de2}) for our algorithm, 
but it is not a priori clear how one can analyze the densities in this case either. 
We develop a novel technique to analyze the densities for our iterative algorithm and prove optimality.  
This technique could be of independent interest to analyzing a broader class of message-passing algorithms.

%
%
\subsection{Performance guarantee and experimental results}
\label{sec:theory}

We provide an upper bound on the probability of error achieved by the iterative inference algorithm and 
task allocation according to the configuration model. 
The bound decays as $e^{-C\taskdegree\quality}$ with a universal constant $C$.
Further, an algorithm-independent lower bound that we establish suggests
that such a dependence of the error on $\taskdegree\quality$ is unavoidable. 

\subsubsection{Bound on the average error probability}
To lighten the notation, let $\hl\equiv \taskdegree-1$ and $\hr\equiv \workerdegree-1$, 
and recall that $\quality=\E[(2\Bp_j-1)^2]$. 
Using these notations, we define $\sigma^2_k$ to be the 
effective variance in the sub-Gaussian tail of our estimates after $k$ iterations of our inference algorithm:
\begin{eqnarray*}
	\sigma_k^2 &\equiv& \frac{2\quality}{\meanquality^2(\quality^2\hl\hr)^{k-1}} \,+\, \Big(3+\frac{1}{\quality\hr}\Big) \frac{1-(1/\quality^2\hl\hr)^{k-1}}{1-(1/\quality^2\hl\hr)}\;.
\end{eqnarray*}
With this, we can prove the following upper bound on the probability of error 
when we run $k$ iterations of our inference algorithm with 
$(\taskdegree,\workerdegree)$-regular assignments on $\ntask$ tasks 
using a crowd with collective quality $\quality$. 
We refer to Section \ref{sec:iterativeproof} for the proof. 
\begin{thm}
	\label{thm:main}
	For fixed $\taskdegree>1$ and $\workerdegree>1$, assume that $\ntask$ tasks are assigned to 
	$\nworker=\ntask\taskdegree/\workerdegree$ workers according to
	a random $(\taskdegree,\workerdegree)$-regular graph drawn from the configuration model. 
	If the distribution of the worker reliability satisfies 
	$\meanquality\equiv\E[2\Bp_j-1]>0$ and $\quality^2 > 1/(\hl\hr)$,  
	then for any $\task\in\{\pm1\}^m$, the estimate after $k$ iterations of 
	the iterative algorithm achieves 
	\begin{eqnarray}
	  \label{eq:main}
	  \frac{1}{\ntask}\sum_{i=1}^\ntask\prob\big(\,\task_i\neq\estimate_i^{(k)}\,\big) &\leq& 
	  e^{-{\taskdegree\quality}/({2\sigma_k^2})} \,+\, \frac{3\taskdegree\workerdegree}{\ntask}(\hl\hr)^{2k-2} \;.
	\end{eqnarray}
\end{thm}
The second term, which is the probability that the resulting graph is not locally tree-like, 
vanishes for large $\ntask$. 
Hence, the dominant term in the error bound is the first term. 
Further, when $\quality^2\hl\hr>1$ as per our assumption 
and when we run our algorithm for large enough number of iterations, 
$\sigma_k^2$ converges linearly to a finite limit $\sigma_\infty^2 \equiv \lim_{k\to \infty} \sigma_k^2$ such that 
\begin{eqnarray*}
	\sigma_\infty^2 &=& \Big(3+\frac{1}{\quality\hr}\Big)\frac{\quality^2\hl\hr}{\quality^2\hl\hr-1} \;.
\end{eqnarray*}
With linear convergence of $\sigma_k^2$, we only need a small number of iterations to achieve $\sigma_k$ close to this limit. 
It follows that for large enough $\ntask$ and $k$, we can prove an upper bound that 
does not dependent on the problem size or the number of iterations, 
which is stated in the following corollary.
\begin{coro}
	\label{cor:mainiterative}
	Under the hypotheses of Theorem \ref{thm:main}, there exists 
	$\ntask_0=3\taskdegree\workerdegree e^{\taskdegree\quality/4\sigma_\infty^2}(\hl\hr)^{2(k-1)}$ 
	and $k_0=1+\big(\log(\quality/\meanquality^2)/\log(\hl\hr\quality^2)\big)$ such that 
	\begin{eqnarray}
	  \label{eq:iterativelimit}
	  \frac{1}{\ntask}\sum_{i=1}^\ntask\prob\big(\,\task_i\neq\estimate_i^{(k)}\,\big) 
	    &\leq& 2e^{ -{\taskdegree\quality}/({4\sigma_\infty^2})}\;,
	\end{eqnarray}
	for all $\ntask\geq\ntask_0$ and $k\geq k_0$.
\end{coro}
\begin{proof}
For $\hl\hr\quality^2>1$ as per our assumption, 
$k=1+\log(\quality/\meanquality^2)/\log(\hl\hr\quality^2)$ iterations suffice 
to ensure that $\sigma_k^2 \leq (2\quality/\meanquality^2)(\hl\hr\quality^2)^{-k+1}+ \quality\hl(3\quality\hr+1)/(\quality^2\hl\hr-1)\leq 2\sigma_\infty^2$.
Also, $\ntask=3\taskdegree\workerdegree e^{\taskdegree\quality/4\sigma_\infty^2}(\hl\hr)^{2(k-1)}$ 
suffices to ensure that $(\hl\hr)^{2k-2}(3\taskdegree\workerdegree)/\ntask\leq
\exp\{-{\taskdegree\quality}/({4\sigma_\infty^2})\}$.
\end{proof}

The required number of iterations $k_0$ is small 
(only logarithmic in $\taskdegree$, $\workerdegree$, $\quality$, and $\meanquality$) 
and does not depend on the problem size $\ntask$. 
On the other hand, the required number of tasks $\ntask_0$ in our main theorem is quite large. 
However, numerical simulations suggest that the actual performance of our approach 
is not very sensitive to the number of tasks and the bound still holds for tasks of small size as well. 
For example, in Figure~\ref{fig:comparison} (left), we ran numerical experiment with $\ntask=1000$, $\quality=0.3$, 
and $k=20$, and the resulting error exhibits exponential decay as predicted by our theorem even for large 
$\taskdegree=\workerdegree=30$. In this case, theoretical requirement on the number of tasks  
$m_0$ is much larger than what we used in the experiment. 

Consider a set of worker distributions $\{\cF\,|\,\E_{\cF}[(2\Bp-1)^2]=\quality\}$ 
that have the same collective quality $\quality$. 
These distributions that have the same value of $\quality$ can 
give different values for $\meanquality$ ranging from $\quality$ to $\quality^{1/2}$.
Our main result on the error rate suggests that 
the error does not depend on the value of $\meanquality$. 
Hence, the effective second moment $\quality$ is 
the right measure of the collective quality of the crowd, 
and the effective first moment $\meanquality$ only affects 
how fast the algorithm converges, 
since we need to run our inference algorithm 
$k=\Omega(1+\log(\quality/\meanquality^2)/\log(\hl\hr\quality^2))$ iterations to  
guarantee the error bound. 

The iterative algorithm is efficient with run-time comparable to 
that of majority voting which requires $O(\ntask\taskdegree)$ operations.
Each iteration of the iterative algorithm requires $O(\ntask\taskdegree)$ operations, 
and we need $O(1+\log(\quality/\meanquality^2)/\log(\quality^2\hl\hr))$ iterations to ensure an 
error bound in \eqref{eq:iterativelimit}. 
By definition, we have $\quality\leq\meanquality\leq\sqrt{\quality}$. 
The run-time is the worst when $\meanquality=\quality$, which happens under 
the spammer-hammer model, and it is the smallest when  $\meanquality=\sqrt{\quality}$ 
which happens if $\Bp_j=(1+\sqrt{\quality})/2$ deterministically. 
In any case, we only need extra logarithmic factor that does not increase with 
compared to majority voting, 
and this
Notice that as we increase the number of iterations, 
the messages converge to an eigenvector of a particular sparse 
matrix of dimensions $\ntask\taskdegree\times\ntask\taskdegree$.  
This suggests that we can alternatively compute the messages 
using other algorithms for computing the top singular vector of large sparse matrices 
that are known to converge faster (e.g. Lanczos algorithm \cite{Lanczos50}). 

Next, we make a few remarks on the performance guarantee. 

First, the assumption that $\meanquality>0$ is necessary. 
If there is no assumption on $\meanquality$, then  
we cannot distinguish if the responses came from 
tasks with $\{\task_i\}_{i\in[m]}$ and workers with $\{p_j\}_{j\in[n]}$ 
or tasks with $\{-\task_i\}_{i\in[m]}$ and workers with $\{1-p_j\}_{j\in[n]}$. 
Statistically, both of them give the same output. 
The hypothesis on $\meanquality$ allows us to distinguish which of the two is the correct solution. 
In the case when we know that $\meanquality<0$, 
we can use the same algorithm  
changing the sign of the final output 
and get the same performance guarantee. 

Second, our algorithm does not require any information on the distribution of $\Bp_j$. 
However, in order to generate a graph that achieves an optimal performance, 
we need the knowledge of $\quality$ for selecting the degree 
$\taskdegree = \Theta(1/\quality\log(1/\error))$.  
Here is a simple way to overcome this limitation at the loss of only additional constant
factor, i.e. scaling of cost per task still remains $\Theta(1/\quality \log (1/\error))$.  
To that end, consider an incremental design in which at step $a$ the 
system is designed assuming $\quality = 2^{-a}$ for $a \geq 1$. At step $a$,
we design two replicas of the task allocation for $\quality=2^{-a}$. 
Now compare the estimates obtained by these two 
replicas for all $\ntask$ tasks. 
If they agree amongst $\ntask(1-2\error)$ tasks, then
we stop and declare that as the final answer. 
Otherwise, we increase $a$ to $a+1$ and repeat. 
Note that by our optimality result, it follows that if $2^{-a}$ is less
than the actual $\quality$ then the iteration must stop with high probability. 
Therefore, the total cost paid is $\Theta(1/\quality \log (1/\error))$ with
high probability. Thus, even lack of knowledge of $\quality$ does not affect
the order optimality of our algorithm.

Further, unlike previous approaches based on Expectation Maximization (EM), 
the iterative algorithm is not sensitive to initialization   
and converges to a unique solution from a random initialization with high probability. 
This follows from the fact that 
the algorithm is essentially computing a leading eigenvector 
of a particular linear operator. 

Finally, we observe a phase transition at $\hl\hr\quality^2=1$. 
Above this phase transition, when $\hl\hr\quality^2>1$, 
we will show that our algorithm is order-optimal and 
the probability of error is significantly smaller than majority voting. 
However, perhaps surprisingly, when we are below the threshold, 
when $\hl\hr\quality^2<1$, we empirically observe that 
our algorithm exhibit a fundamentally different behavior (cf. Figure~\ref{fig:comparison}). 
The error we get after $k$ iterations of our algorithm 
increases with $k$. 
In this regime, we are better off stopping the algorithm after 1 iteration, 
in which case the estimate we get is essentially the 
same as the simple majority voting,  
and we cannot do better than majority voting. 
This phase transition is universal and we observe similar behavior 
with other inference algorithms including EM approaches. 
We provide more discussions on the choice of $\taskdegree$ 
and the limitations of having small $\workerdegree$ in the following section.

%
%
\subsubsection{Minimax optimality of our approach} 
For a task master, the natural core optimization problem of her concern is 
how to achieve a certain reliability in the answers with minimum cost. 
Throughout this paper, we assume that the cost is proportional to the total number of queries. 
In this section, we show that if a taskmaster wants to achieve a target error rate of $\error$, 
she can do so using our approach with budget per task scaling as $O((1/\quality)\log(1/\error))$ 
for a broad range of worker degree $\workerdegree$. 
Compared to the necessary condition which we provide in Section~\ref{sec:minimax_nonadaptive}, 
this is within a constant factor from 
what is necessary using the best {\em non-adaptive} task assignment and the best inference algorithm.
Further, we show in Section~\ref{sec:minimax_adaptive} that this scaling in the budget is still necessary 
if we allow using the best {\em adaptive} task assignment together with the best inference algorithm.
This proves that our approach is minimax optimal up to a constant factor in the budget. 

Assuming for now that there is no restrictions on the worker degree $\workerdegree$ 
and we can assign as many tasks to each worker as we want, 
we can get the following simplified upper bound on the error 
that holds for all $\workerdegree\geq1+1/\quality$. 
To simplify the resulting bound, let us assume for now that $\hl\hr\quality\geq2$. 
Then, we get that $\sigma_\infty^2 \leq 2(3+1/\hr\quality)$. 
Then from \eqref{eq:iterativelimit}, we get the following bound:
\begin{eqnarray*}
	\frac{1}{\ntask}\sum_{i\in[\ntask]}\prob(\task_i\neq\estimate_i^{(k)}) &\leq& 2e^{-\taskdegree\quality/32}\;,
\end{eqnarray*}
for large enough $\ntask\geq\ntask_0$.
In terms of the budget or the number of queries necessary to achieve a target accuracy, 
we get the following sufficient condition as a corollary.
\begin{coro}
	\label{cor:budgetiterative}
	Using the non-adaptive task assignment scheme with $\workerdegree\geq1+1/\quality$ 
	and the iterative inference algorithm 
	introduced in Section~\ref{sec:algorithm}, it is sufficient to query 
	$(32/\quality)\log(2/\error)$ times per task to guarantee that 
	the probability of error is at most $\error$ for any $\error\leq1/2$ and for all $m\geq m_0$.   
\end{coro}
We provide a matching minimax necessary condition up to a constant factor for non-adaptive algorithms in Section~\ref{sec:minimax_nonadaptive}. 
When the nature can choose the worst-case worker distributions, 
no non-adaptive algorithm can achieve error less than $\error$ with budget per task smaller than 
$(C'/\quality)\log(1/2\error)$ with some universal positive constant $C'$. 
This establishes that under the non-adaptive scenario,  
our approach is minimax optimal up to a constant factor for large enough $\ntask$. 
With our approach you only need to ask (and pay for) 
a constant factor more than what is necessary 
using the best non-adaptive task assignment scheme 
together with the best inference algorithm under the worst-case worker distribution. 
 
Perhaps surprisingly, we will show in Section~\ref{sec:minimax_adaptive} 
that the necessary condition does not change even if we allow adaptive task assignments. 
No algorithm, adaptive or non-adaptive, can achieve error less than $\error$ 
without asking $(C''/\quality)\log(1/2\error)$ queries per task with some universal positive constant $C''$.  
Hence, our non-adaptive approach achieves minimax optimal performance that can be achieved by the best adaptive scheme.
 
In practice, we might not be allowed to have large $\workerdegree$ depending on the application. 
For different regimes of the restrictions on the allowed worker degree $\workerdegree$, 
we need different choices of $\taskdegree$. 
When we have a target accuracy $\error$, the following corollary establishes that we can achieve probability of error $\error$ with $\taskdegree\geq C(1+1/\hr\quality)(1/\quality)\log(1/\error)$ for any value of $\workerdegree$. 
\begin{coro}
	\label{cor:budgetiterative2}
	Using the non-adaptive task assignment scheme with any $\workerdegree$ 
	and the iterative inference algorithm 
	introduced in Section~\ref{sec:algorithm}, it is sufficient to query 
	$(24+8/\hr\quality)(1/\quality)\log(2/\error)$ times per task 
	to guarantee that 
	the probability of error is at most $\error$ for any $\error\leq1/2$ and for all $m\geq m_0$. 
\end{coro}
\begin{proof}
We will show that for $\taskdegree\geq\max\{1+2/(\hr\quality^2)\,,\,8(3+1/\hr\quality)(1/\quality)\log(1/\error)\}$, the probability of error is at most $\error$. 
Since, $1+2/(\hr\quality^2) \leq 8(3+1/\hr\quality)(1/\quality)\log(1/\error)$ for $\error\leq 1/2$, this proves the corollary. 
Since $\hl\hr\quality^2\geq 2$ from the first condition, 
we get that $\sigma_\infty^2 \leq 2(3+1/\hr\quality)$. 
Then, the probability of error is upper bounded by 
$2\exp\{-\taskdegree\quality/(24+8/\hr\quality)\}$.
This implies that for $\taskdegree\geq (24+8/\hr\quality)(1/\quality)\log(2/\error)$ 
the probability of error is at most $\error$. 
\end{proof}

For $\workerdegree\geq C'/\quality$, this implies that 
our approach requires $O((1/\quality)\log(1/\error))$ queries and it is minimax optimal. 
However, for $\workerdegree=O(1)$, 
our approach requires $O((1/\quality^2)\log(1/\error))$ queries. 
This is due to the fact that when $\workerdegree$ is small, 
we cannot efficiently learn the quality of the workers  
and need significantly more questions to achieve the accuracy we desire. 
Hence, in practice, we want to be able to assign more tasks to each worker 
when we have low-quality workers. 

%
%
\subsubsection{Experimental results}

Figure.~\ref{fig:comparison} shows the comparisons between probabilities of error 
achieved by different inference algorithms, but on the same task assignment using 
regular bipartite random graphs. 
We ran $20$ iterations of EM and our iterative algorithm, 
and also the spectral approach of using leading left singular vector of $A$ for estimation.  
The spectral approach, which we call Singular Vector in the graph, 
is explained in detail in Section~\ref{sec:relationlowrank}.
The error rates are compared with those of majority voting and the oracle estimator. 
The oracle estimator performance sets a lower bound on what any inference algorithm can achieve, 
since it knows all the values of $p_j$'s.
For the numerical simulation on the left-hand side, we set $\ntask=1000$, 
$\taskdegree=\workerdegree$ and 
used the spammer hammer model for the distribution of the workers with $\quality=0.3$. 
According to our theorem, we expect a phase transition at $\taskdegree=1+1/0.3=4.3333$. 
{ From the figure, we observe that the iterative inference algorithm starts to perform better than 
majority voting at $\taskdegree=5$.}
For the figure on the right-hand side, we set $\taskdegree=25$. 
For fair comparisons with the EM approach, 
we used an implementation of the EM approach in Java by Sheng et al. \cite{SPI08}, which is publicly available.
 
\begin{figure}
\begin{center}
	\includegraphics[width=.45\textwidth]{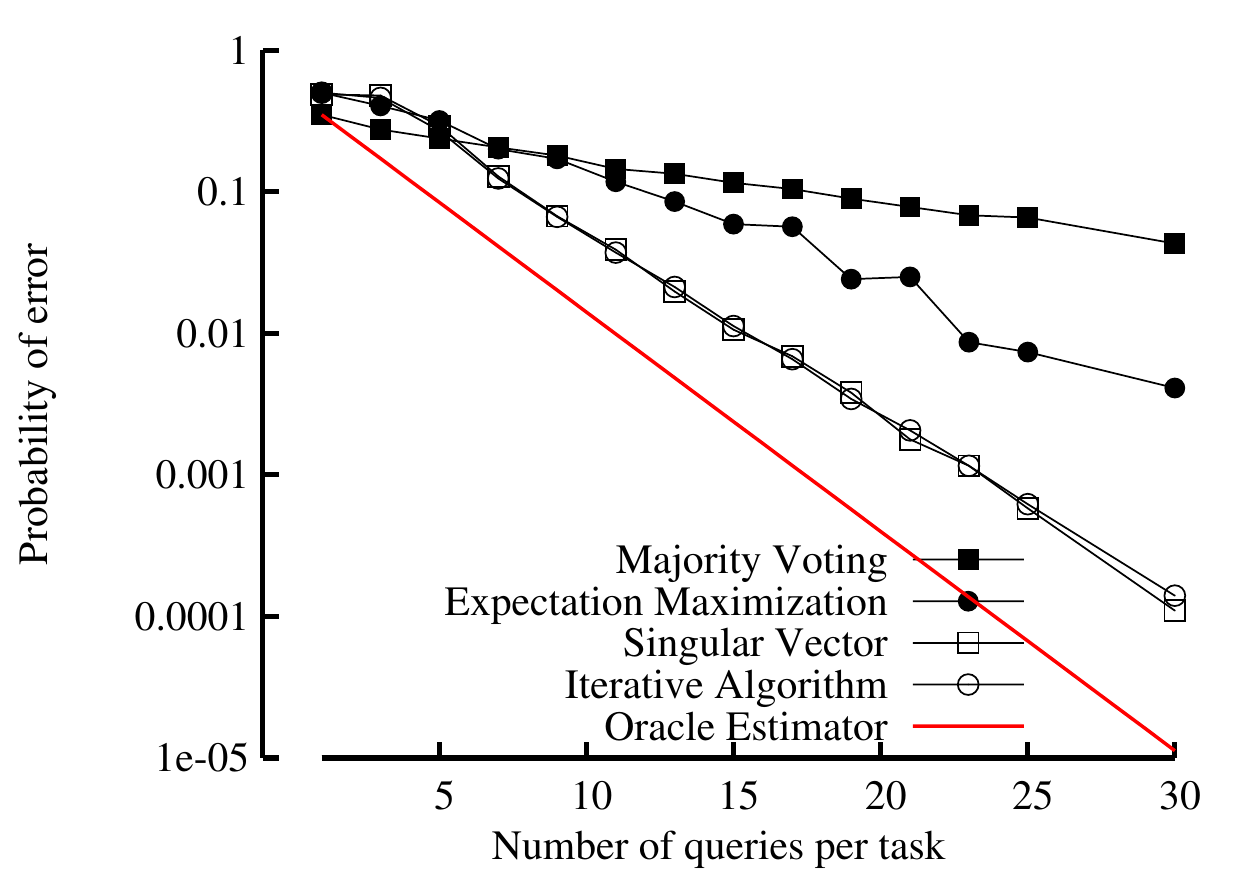}
	\includegraphics[width=.45\textwidth]{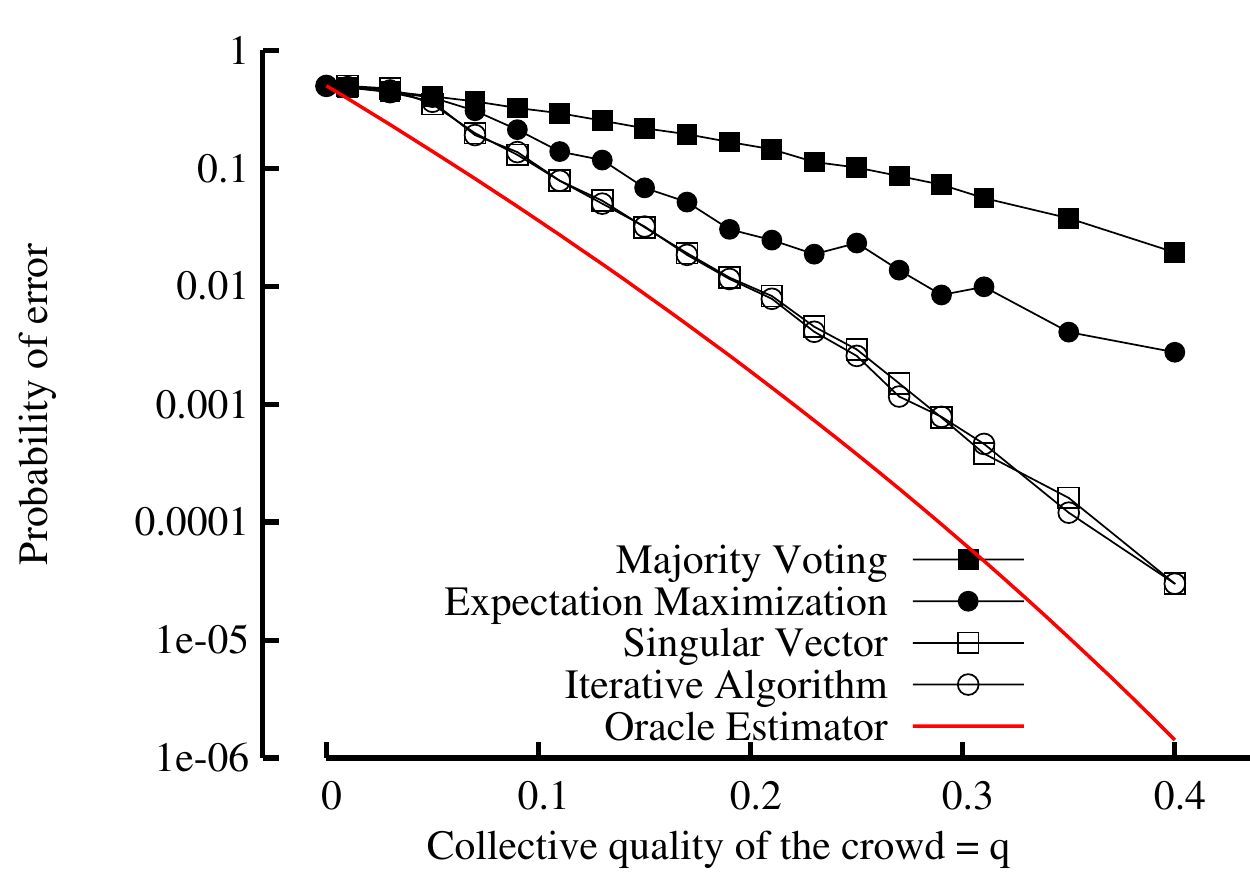}
\end{center}
\caption{The iterative algorithm improves over majority voting and EM algorithm. 
Using the top singular vector for inference has similar performance as our iterative approach.} 
\label{fig:comparison}
\end{figure}

We also ran two experiments with real crowd using Amazon Mechanical Turk. 
In our experiments, we created tasks for comparing colors; 
we showed three colors on each task, one on the top and two on the bottom. 
We asked the crowd to indicate ``if the color on the top is more similar to the color on the left or on the right''.

The first experiment confirms that the ground truth for these color comparisons tasks 
are what is expected from pairwise distances in the Lab color space.  
The distances in the Lab color space between the a pair of colors are known to be a good measure of the perceived distance between the pair \cite{COLOR}. 
To check the validity of this Lab distance 
we collected $210$ responses on each of the $10$ color comparison tasks. 
As shown in Figure.~\ref{fig:comparison1}, 
for all $10$ tasks, the majority of the $210$ responses were consistent with 
the Lab distance based ground truth. 

Next, to test our approach, we created $50$ of such similarity tasks and 
recruited $28$ workers to answer all the questions. 
Once we have this data, we can subsample the data to simulate 
what would have happened if we collected smaller number of responses per task. 
The resulting average probability of error is illustrated in Figure.~\ref{fig:comparison2}.
For this crowd from Amazon Mechanical Turk, we can estimate the collective quality from the data, which is about $\quality\simeq0.175$. 
Theoretically, this indicates that phase transition should happen when $(\taskdegree-1)((50/28)\taskdegree-1)\quality^2=1$, 
since we set $\workerdegree=(50/28)\taskdegree$. 
With this, we expect phase transition to happen around $\taskdegree\simeq5$. 
In Figure.~\ref{fig:comparison2}, we see that our iterative algorithm starts to perform better 
than majority voting around $\taskdegree=8$. 
\begin{figure}
\begin{center}
	\includegraphics[width=15.5cm]{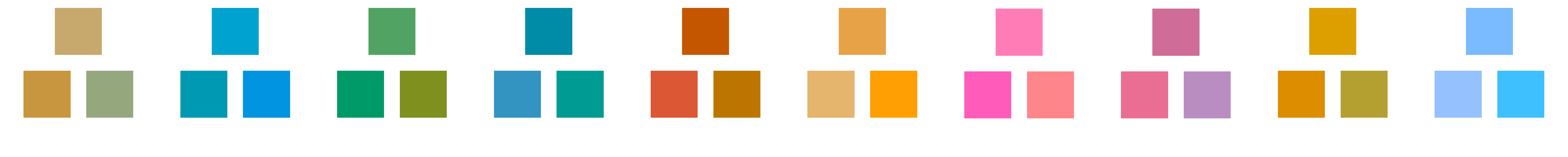}
	\put(-435,-1){\small$151$}
	\put(-413,-1){\small$59$}
	\put(-392,-1){\small$123$}
	\put(-370,-1){\small$87$}
	\put(-348,-1){\small$141$}
	\put(-326,-1){\small$69$}
	\put(-304,-1){\small$126$}
	\put(-282,-1){\small$84$}
	\put(-260,-1){\small$109$}
	\put(-239,-1){\small$101$}
	\put(-215,-1){\small$121$}
	\put(-193,-1){\small$89$}
	\put(-170,-1){\small$141$}
	\put(-149,-1){\small$69$}
	\put(-126,-1){\small$141$}
	\put(-105,-1){\small$69$}
	\put(-81,-1){\small$149$}
	\put(-61,-1){\small$61$}
	\put(-40,-1){\small$159$}
	\put(-17,-1){\small$51$}
	\end{center}
	\caption{Experimental results on color comparison using real data from Amazon's Mechanical Turk. 
	The color on the left is closer to the one on the top in Lab distance for each triplet. 
	The votes from $210$ workers are shown below each triplet.}
	\label{fig:comparison1}
\end{figure}

\begin{figure}
	\begin{center}
		\includegraphics[width=9.5cm]{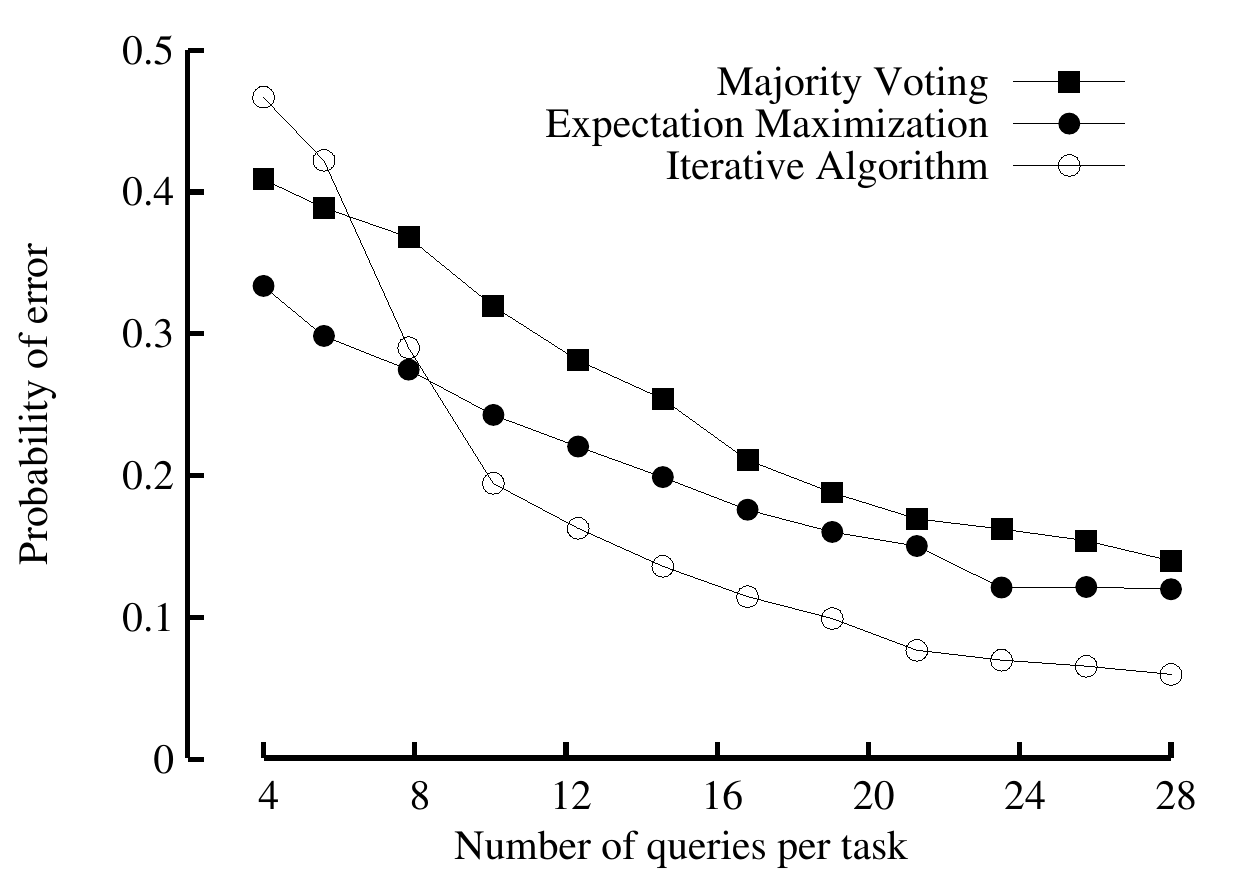}
	\end{center}
	\caption{The average probability of error on color comparisons using real data from Amazon's Mechanical Turk.}
	\label{fig:comparison2}
\end{figure}

%
%
\subsection{Fundamental limit under the non-adaptive scenario}
\label{sec:minimax_nonadaptive}

Under the non-adaptive scenario, we are allowed to use only 
non-adaptive task assignment schemes which assign all the tasks a priori 
and collect all the responses simultaneously. 
In this section, we investigate the fundamental limit on how small an error can be achieved 
using the best possible non-adaptive task assignment scheme together with the 
best possible inference algorithm. 
In particular, we are interested in the minimax optimality: 
What is the minimum error that can be achieved under 
the worst-case worker distribution?
To this end, we analyze the performance of an oracle estimator 
when the workers' latent qualities are drawn from a specific distribution 
and provide a lower bound on the minimax rate on the probability of error. 
Compared to our main result, this establishes that 
our approach is minimax optimal up to a constant factor.  

In terms of the budget, 
the natural core optimization problem of our concern is 
how to achieve a certain reliability in our answers with minimum cost. 
Let us assume that the cost is proportional to the total number of queries. 
We show that 
for a given target error rate $\error$, 
the total budget sufficient to achieve this target error rate using our algorithm  
is within a constant factor from 
what is necessary using the best {non-adaptive} task assignment and the best inference algorithm. 

\vspace{0.2cm}
\noindent{\bf Fundamental limit.} 
Consider a crowd characterized by worker distribution $\cF$ such that $\Bp_j\sim\cF$. 
Let $\cF_\quality$ be a set of all distributions on $[0,1]$, such that the collective 
quality is parametrized by $\quality$: 
\begin{eqnarray*}
	\cF_\quality &=& \big\{ \cF\;|\; \E_\cF[(2\Bp_j-1)^2]=\quality \big\}\;.
\end{eqnarray*}
We want to prove a lower bound on the {\em minimax rate} on the probability of error, 
which only depends on $\quality$ and $\taskdegree$. 
Define the minimax rate as 
\begin{eqnarray*}
	\min_{\tau\in\cT_\taskdegree,\estimate}\;\; \max_{\task\in\{\pm1\}^\ntask,\cF\in\cF_\quality}\;\; \frac{1}{\ntask}\sum_{i\in[\ntask]}\prob\big(\task_i\neq\estimate_i\big)\;,
\end{eqnarray*}
where $\estimate$ ranges over all estimators which are measurable functions over the responses, 
and $\tau$ ranges over the set $\cT_\taskdegree$ of all task assignment schemes 
which are non-adaptive and ask $\ntask\taskdegree$ queries in total.
Here the probability is taken over all realizations of 
$\Bp_j$'s, $\BA_{ij}$'s, and the randomness introduced in the task assignment and the inference. 

Consider any non-adaptive scheme that assigns $\taskdegree_i$ workers to the $i$-th task. 
The only constraint is that the average number of queries is bounded by 
$(1/\ntask)\sum_{i\in[\ntask]}\taskdegree_i\leq\taskdegree$. 
To get a lower bound on the minimum achievable error, 
we consider an oracle estimator that has access to all the $\Bp_j$'s, and hence can make an optimal estimation.  
Further, since we are proving minimax optimality and not instance-optimality, 
the worst-case error rate will always be lower bounded by 
the error rate for any choice of worker distribution.
In particular, we prove a lower bound using the spammer-hammer model. 
Concretely, we assume the $\Bp_j$'s are drawn from the spammer-hammer model with perfect hammers: 
\begin{eqnarray*}
 \Bp_j &=& \left\{ 
      \begin{array}{rl} 1/2 &\text{ with probability }1-\quality \;,\\
      1 &\text{ otherwise} \;. \end{array}
	      \right.
\end{eqnarray*}
Notice that the use of $\quality$ is consistent with $\E[(2\Bp_j-1)^2]=\quality$.
Under the spammer-hammer model, 
the oracle estimator only makes a mistake on task $i$ if it is only assigned to spammers, 
in which case we flip a fair coin to achieve error probability of half. Formally, 
\begin{eqnarray*}
	\prob(\estimate_i\neq \task_i)&=&\frac12(1-\quality)^{\taskdegree_i}\;. 
\end{eqnarray*}
By convexity and using Jensen's inequality, the average probability of error is lower bounded by 
\begin{eqnarray*}
	\frac1\ntask\sum_{i\in[\ntask]}\prob(\estimate_i\neq \task_i)&\geq&\frac12(1-\quality)^{\taskdegree}\;. 
\end{eqnarray*}
Since we are interested in how many more queries are necessary as the quality of the crowd deteriorates, 
we are going to assume $\quality\leq2/3$, 
in which case $(1-\quality)\geq e^{-(\quality+\quality^2)}$. 
As long as total $\ntask\taskdegree$ queries are used, 
this lower bound holds regardless of how the actual tasks are assigned. 
And since this lower bound holds for a particular choice of $\cF$, it holds 
for the worst case $\cF$ as well. 
Hence, for the best task assignment scheme and the best inference algorithm, we have  
\begin{eqnarray*}
	\min_{\tau\in\cT_\taskdegree,\estimate}\;\; \max_{\task\in\{\pm1\}^\ntask,\cF\in\cF_\quality}\;\; \frac{1}{\ntask}\sum_{i\in[\ntask]}\prob\big(\task_i\neq\estimate_i\big) &\geq& \frac12e^{-(\quality+\quality^2)\taskdegree}\;.
\end{eqnarray*}
This lower bound on the minimax rate holds for any positive integer $\ntask$, and 
regardless of the number of workers or the number of queries, $\workerdegree$, assigned to each worker. 
In terms of the average number of queries necessary to achieve 
a target accuracy of $\error$, this implies the following necessary condition. 
\begin{lemma}
	\label{lem:budgetoracle}
	Assuming $\quality\leq2/3$ and the {non-adaptive} scenario, 
	if the average number of queries per task is less than 
	$(1/2\quality)\log(1/2\error)$, then no algorithm can achieve 
	average probability of error less than $\error$ for any $\ntask$ under the worst-case worker distribution. 
\end{lemma}
To prove this worst-cased bound, we analyzed a specific distribution of the spammer-hammer model. 
However, the result (up to a constant factor) seems to be quite general and can also 
be proved using different distributions, e.g. when all workers have the same quality. 
The assumption on $\quality$ can be relaxed as much as we want, 
by increasing the constant in the necessary budget. 
Compared to the sufficient condition in Corollary~\ref{cor:budgetiterative}
this establishes that our approach is minimax optimal up to a constant factor. 
With our approach you only need to ask (and pay for) 
a constant factor more than what is necessary for any algorithm.

\vspace{0.2cm}
\noindent{\bf Majority voting.} 
As a comparison, we can do similar analysis for the simple majority voting 
and show that the performance is significantly worse than our approach. 
The next lemma provides a bound on the minimax rate of majority voting. 
A proof of this lemma is provided in Section~\ref{sec:proof_majority}. 
\begin{lemma}
	\label{lem:majority}
	For any $C<1$, there exists a positive constant $C'$ such that 
	when $\quality\leq C$, the error achieved by majority voting is at least  
	\begin{eqnarray*}
		\label{eq:majorityerror}
		\min_{\tau\in\cT_\taskdegree}\;\; \max_{\task\in\{\pm1\}^\ntask,\cF\in\cF_\quality}\;\; \frac{1}{\ntask} \sum_{i\in[\ntask]} \prob\big(\task_i\neq\estimate_i\big) &\geq& e^{-C'(\taskdegree\quality^2+1)} \;. 
	\end{eqnarray*}
\end{lemma}
In terms of the number of queries necessary to achieve a target accuracy $\error$ 
using majority voting, this implies that 
we need to ask at least $(c/\quality^2)\log(c'/\error)$ 
queries per task for some universal constants $c$ and $c'$.
Hence, majority voting is significantly more costly than our approach in terms of budget. 
Our algorithm is more efficient in terms of computational complexity as well. 
Simple majority voting requires $O\big((\ntask/\quality^2)\log(1/\error)\big)$ operations 
to achieve target error rate $\error$ in the worst case. 
From Corollary~\ref{cor:mainiterative}, together with $\taskdegree=O((1/\quality)\log(1/\error))$ 
and $\taskdegree\workerdegree\quality^2=\Omega(1)$, we get that  
our approach requires $O((\ntask/\quality)\log(1/\quality)\log(1/\error))$ operations 
in the worst case.

%
%
\subsection{Fundamental limit under the adaptive scenario}
\label{sec:minimax_adaptive}

In terms of the scaling of the budget necessary to achieve a target accuracy, 
we established that using a {\em non-adaptive} task assignment, 
no algorithm can do better than our approach. 
One might prefer a non-adaptive scheme in practice  because having all the batches
of tasks processed in parallel reduces the latency. 
This is crucial in many applications, 
especially in real-time applications such as searching, visual information processing, 
and document processing \cite{BJJ10,BLM10,YKG10,BBMK11}. 
However, by switching to an adaptive task assignment, 
one might hope to be more efficient and
still obtain a desired accuracy from fewer questions. 
On one hand, adaptation can help improve performance. 
But on the other hand, it can significantly complicate system design due to careful 
synchronization requirements. 
In this section, we want to prove an algorithm-independent  upper bound on how much 
one can gain by using an adaptive task allocation. 

When the identities of the workers are known, 
one might be tempted to first identify which workers are more reliable and 
then assign all the tasks to those workers in an explore/exploit manner. 
However, in typical crowdsourcing platforms such as Amazon Mechanical Turk, it is unrealistic to assume that 
we can identify and reuse any particular worker, 
since typical workers are neither persistent nor identifiable and 
batches are distributed through an open-call. 
Hence, exploiting a reliable worker is not possible. 
However, we can adaptively resubmit batches of tasks; 
we can dynamically choose which subset of tasks 
to assign to the next arriving worker. 
In particular, we can allocate tasks to the next batch 
based on all the information we have on all the tasks from the responses collected thus far. 
For example, one might hope to reduce uncertainty more efficiently 
by adaptively collecting more responses on those tasks that she is less certain about. 

\vspace{0.2cm}
\noindent{\bf Fundamental limit.} 
In this section, we show that, perhaps surprisingly, there is no significant gain 
in switching from our non-adaptive approach to an adaptive strategy when the workers are {\em fleeting}. 
We first prove a lower bound on the minimax error rate: the error that is achieved 
by the best inference algorithm $\estimate$ using the best adaptive task allocation scheme $\tau$ 
under a worst-case worker distribution $\cF$ and the worst-case true answers $\task$. 
Let $\ctT_\taskdegree$ be the set of all task assignment schemes that use at most $\ntask\taskdegree$ queries in total. 
Then, we can show the following lower bound on the minimax rate on the probability of error. 
A proof of this theorem is provided in Section \ref{sec:proof_minimax_adaptive}. 
\begin{thm}
	\label{thm:minimax_adaptive} 
	When $\quality\leq C$ for any constant $C<1$, there exists a 
	positive constant $C'$ such that 
	\begin{eqnarray}
		\min_{\tau\in\ctT_\taskdegree,\estimate}\;\; \max_{\task\in\{\pm1\}^\ntask,\cF\in\cF_\quality}\;\; \frac{1}{\ntask}\sum_{i\in[\ntask]}\prob\big(\task_i\neq\estimate_i\big) &\geq&\frac12e^{-C'\,\taskdegree\quality} \;,
		\label{eq:minimax_adaptive}
	\end{eqnarray}
	for all $\ntask$ 
	where the task assignment scheme $\tau$ ranges over all adaptive schemes that use at most $\ntask\taskdegree$ 
	queries and $\estimate$ ranges over all estimators that are measurable functions over the responses.  
\end{thm}
We cannot avoid the factor of half in the lower bound, since
we can always achieve error probability of half without asking any queries (with $\taskdegree=0$). 
In terms of the budget required to achieve a target accuracy, 
the above lower bound proves that no algorithm, adaptive or non-adaptive, 
can achieve an error rate less than $\error$ with number of queries per task 
less than $(C'/\quality)\log(2/\error)$ in the worst case of worker distribution. 
\begin{coro}
      \label{cor:minimax_adaptive}
	Assuming $\quality\leq C$ for any constant $C<1$ and 
	the {iterative} scenario, there exists a positive constant $C'$ such that 
	if the average number of queries is less than 
	$(C'/\quality)\log(1/2\error)$, then no algorithm can achieve 
	average probability of error less than $\error$ for any $\ntask$ 
	under the worst-case worker distribution. 
\end{coro}
Compared to Corollary~\ref{cor:budgetiterative}, 
we have a matching sufficient and necessary conditions up to a constant factor. 
This proves that there is no significant gain in using an adaptive scheme, 
and our approach achieves minimax-optimality up to a constant factor with a non-adaptive scheme. 
This limitation of adaptation strongly relies on the fact that workers are {\em fleeting}
in existing platforms and can not be reused. 
Therefore, architecturally our results suggest that 
building a reliable reputation system for workers would be 
essential to harnessing the potential of adaptive designs. 

\vspace{0.2cm}
\noindent{\bf A counter example for instance-optimality. } 
The above corollary establishes minimax-optimality: 
for the worst-case worker distribution, 
no algorithm can improve over our approach other 
than improving the constant factor in the necessary budget. 
However, this does not imply instance-optimality. 
In fact, there exists a family of worker distributions 
where all non-adaptive algorithms fail to achieve order-optimal performance 
whereas a trivial adaptive algorithm succeeds. 
Hence, for particular instances of worker distributions, 
there exists a gap between what can be achieved using 
non-adaptive algorithms and adaptive ones. 

We will prove this in the case of the spammer-hammer model 
where each new worker is a hammer ($p_j=1$) with probability $\quality$ 
or a spammer ($p_j=1/2$) otherwise.   
We showed in Section~\ref{sec:minimax_nonadaptive} that 
no non-adaptive algorithm can achieve an error less than 
$(1/2)e^{-C'\taskdegree\quality}$ for any value of $\ntask$. 
In particular, this does not vanish even if we increase $\ntask$. 
We will introduce a simple adaptive algorithm and show that this algorithm  
achieves an error probability that goes to zero as $\ntask$ grows. 

The algorithm first groups all the tasks into $\sqrt{\ntask}$ disjoint sets of size $\sqrt{\ntask}$ each. 
Starting with the first group, 
the algorithm assigns all $\sqrt{\ntask}$ tasks to new arriving workers until 
it sees two workers who agreed on all $\sqrt{\ntask}$ tasks. 
It declares those responses as its estimate for this group and 
moves on to the next group. 
This process is repeated until it reaches the allowed number of queries. 
This estimator makes a mistake on a group if $(a)$ there were two spammers who agreed on all $\sqrt{\ntask}$ tasks 
or $(b)$ we run out of allowed number of queries before we finish the last group. Formally, we can prove the following upper bound on the probability of error. 
\begin{lemma}
	\label{lem:counterexample}
	Under the spammer-hammer model, 
	when the allowed number of queries per task $\taskdegree$ is larger than $2/\quality$, 
	there is an adaptive task allocation scheme and an inference algorithm that achieves 
	average probability of error at most $\ntask\taskdegree^2 2^{-\sqrt{\ntask}}+e^{-(2/\taskdegree)(\taskdegree\quality-2)^2\sqrt{\ntask}}$. 
\end{lemma}
\begin{proof} 
Recall that we are only allowed $\taskdegree\ntask$ queries. 
Since we are allocating $\sqrt\ntask$ queries per worker, 
we can only ask at most $\taskdegree\sqrt{\ntask}$ workers. 
First, the probability that there is at least one pair of spammers 
(among all possible pairs from $\taskdegree\sqrt\ntask$ workers)
who agreed an all $\sqrt\ntask$ responses is  
at most $\ntask\taskdegree^2 2^{-\sqrt\ntask}$.
Next, given that no pairs of spammers agreed on all their responses, 
the probability that we run out of all $\ntask\taskdegree$ allowed queries is 
the probability that the number of hammers in $\taskdegree\sqrt\ntask$ workers 
is strictly less than $2\sqrt\ntask$ (which is the number of hammers 
we need in order to terminate the algorithm, 
conditioned on that no spammers agree with one another). 
By standard concentration results, 
this happens with probability at most 
$e^{-(2/\taskdegree)(\taskdegree\quality-2)^2\sqrt{\ntask}}$. 
\end{proof}

This proves the existence of an adaptive algorithm 
which achieves vanishing error probability as $\ntask$ grows for a board range of task degree $\taskdegree$. 
Comparing the above upper bound with the known lower bound for non-adaptive schemes, 
this proves that non-adaptive algorithms cannot be instance optimal: 
there is a family of distributions where 
adaptation can significantly improve performance.
This is generally true when there is a strictly positive probability 
that a worker is a hammer ($p_j=1$). 

One might be tempted to apply the above algorithm 
in more general settings other than the spammer-hammer model. 
However, this algorithm fails when there are no perfect workers in the crowd. 
If we apply this algorithm in such a general setting, 
then it produces useless answers: 
the probability of error approaches half as $\ntask$ grows for any finite $\taskdegree$.


\subsection{Connections to low-rank matrix approximation}
\label{sec:relationlowrank}

In this section, we first explain why the top singular vector of the data matrix $\BA$ 
reveals the true answers of the tasks, 
where $\BA$ is the $\ntask\times\nworker$ matrix of the responses and 
we fill in zeros wherever we have no responses collected. 
This naturally defines a spectral algorithm for inference which we present next. 
It was proven in \cite{KOS11allerton} 
that the error achieved by this spectral algorithm is upper bounded by 
$C/(\taskdegree\quality)$ with some constant $C$. 
But numerical experiments (cf. Figure~\ref{fig:comparison}) suggest that the error decays much faster, 
and that the gap is due to the weakness of the analysis used in \cite{KOS11allerton}. 
Inspired by this spectral approach, 
we introduced a novel inference algorithm that performs as well as the spectral algorithm (cf. Figure~\ref{fig:comparison}) and proved a much tighter upper bound on the resulting error which 
scales as $e^{-C'\taskdegree\quality}$ with some constant $C'$. 
Our inference algorithm is based on {\em power iteration}, 
which is a well-known algorithm for computing the top singular vector of a matrix, 
and Figure~\ref{fig:comparison} suggests that both algorithms are equally effective  
and the resulting errors are almost identical.  


The data matrix $\BA$ can be viewed as a rank-$1$ matrix that is perturbed by random noise. 
Since, $\E[\BA_{ij}|\task_i,\Bp_j]=(\workerdegree/\ntask)\task_i(2\Bp_j-1)$, 
the conditional expectation of this matrix is  
\begin{eqnarray*}
	\E\big[\BA\,|\,\task,\Bp\big] &=& \Big(\frac{\workerdegree}{\ntask}\Big)\task(2\Bp-\ones)^T\;,
\end{eqnarray*}
where $\ones$ is the all ones vector, 
the vector of correct solutions is $\task=\{\task_i\}_{i\in[\ntask]}$ and 
the vector of worker reliability is $\Bp=\{\Bp_j\}_{j\in[\nworker]}$. 
Notice that the rank of this conditional expectation matrix is one and 
this matrix reveals the correct solutions exactly. 
We can decompose $\BA$ into a low-rank expectation plus a random perturbation: 
\begin{eqnarray*}
	\BA &=& \Big(\frac{\workerdegree}{\ntask}\Big)\task(2\Bp-\ones)^T + \BZ\;,
\end{eqnarray*}
where $\BZ\equiv\BA-\E\big[\BA\,|\,\task,\Bp\big]$ is the random perturbation with zero mean.
When the spectral radius of the noise matrix $\BZ$ is much smaller than the spectral radius of the signal, 
we can correctly extract most of $\task$ using the leading left singular vector of $\BA$. 

Under the crowdsourcing model considered in this paper, 
an inference algorithm using the top left singular vector of $\BA$ 
was introduced and analyzed by Karger et al. \cite{KOS11allerton}. 
Let $u$ be the top left singular vector of $\BA$. 
They proposed estimating $\estimate_i = \sign(u_i)$ 
and proved an upper bound on the probability of error that scales as $O(1/\taskdegree\quality)$. 
The main technique behind this result is 
in analyzing the spectral gap of $\BA$. 
It is not difficult to see that the spectral radius of 
the conditional expectation matrix is $(\workerdegree/\ntask)\|\task(2\Bp-\ones)^T\|_2 = \sqrt{\taskdegree\workerdegree\quality}$, 
where the operator norm of a matrix is denoted by $\|X\|_2\equiv \max_a \|X a\|/\|a\|$. 
Karger et al. proved that the spectral radius of the perturbation $\|\BZ\|_2$ is 
in the order of $(\taskdegree\workerdegree)^{1/4}$. 
Hence, when $\taskdegree\workerdegree\quality^2\gg1$, 
we expect a separation between the conditional expectation and the noise. 

One way to compute the leading singular vector is to use power iteration:  
for two vectors $u\in\reals^\ntask$ and $v\in\reals^n$, 
starting with a randomly initialized $v$, power iteration iteratively updates 
$u$ and $v$ by repeating $u=Av$ and $v=A^Tu$. 
It is known that normalized $u$ (and $v$) converges linearly to the leading left (and right) singular vector. 
Then we can use the sign of $u_i$ to estimate $\task_i$. 
Writing the update rule for each entry, we get 
\begin{eqnarray*}
	u_i=\sum_{j\in\partial i} A_{ij}v_j ,\;\;\;\;\;\;\;\;\;\; v_j=\sum_{i\in\partial j} A_{ij}u_i \;\;.  
\end{eqnarray*}
Notice that this power iteration update rule is  
almost identical to those of message passing updates in \eqref{eq:messageupdate1} 
and \eqref{eq:messageupdate2}. 
The $\taskdegree$ task messages $\{x_{i\to j}\}_{j\in\partial i}$ from task $i$ are close in value to 
the entry $u_i$ of the power iteration. 
The $\workerdegree$ worker messages $\{y_{j\to i}\}_{i\in\partial j}$ from worker $j$ are close in value to 
the entry $v_j$ of the power iteration. 
Numerical simulations in Figure~\ref{fig:comparison} suggest that 
the quality of the estimates from the two algorithms are almost identical. 
However, the known performance guarantee for the spectral approach is weak. 
We developed novel analysis techniques to analyze our message passing algorithm, 
and provide an upper bound on the error that scales as $e^{-C\taskdegree\quality}$. 
It might be possible to apply our algorithm, together with the analysis techniques, 
to other problems where the top singular vector of a data matrix is used for inference.

\subsection{Connections to belief propagation}
\label{sec:relationbp}
The crowdsourcing model described in this paper can naturally be described using a graphical model. 
Let $G([\ntask]\times[\nworker],E,A)$ denote the weighted bipartite graph, where 
$[\ntask]$ is  the set of $\ntask$ task nodes, $[\nworker]$ is the set of $\nworker$ worker nodes, 
$E$ is the set of edges connecting a task to a worker who is assigned that task, 
and $A$ is the set of weights on those edges according to the responses. 
Given such a graph, we want to find a set of task answers that 
maximize the following posterior distribution 
$\bpdist(\estimate,p):\{\pm1\}^\ntask\times[0,1]^\nworker\to\reals^+$.  
\begin{eqnarray*}
	\max_{\estimate,p} \;\;\;\prod_{a\in[\nworker]}\crowddist(p_a)\; \prod_{(i,a)\in E}\Big\{p_a\ind(\estimate_i=A_{ia})+(1-p_a)\ind(\estimate_i\neq A_{ia})\Big\}\;,
\end{eqnarray*}
where with a slight abuse of notation we use $\crowddist(\cdot)$ to denote the prior probability density function of $p_a$'s and we use $i$ and $j$ to denote task nodes and $a$ and $b$ to denote worker nodes. 
For such a problem of finding the most probable realization in a graphical model,  
the celebrated belief propagation (BP) gives a good approximate solution. 
To be precise, BP is an approximation for maximizing the marginal distribution of each variable, 
and a similar algorithm known as min-sum algorithm approximates the most probable realization. 
However, the two algorithms are closely related, and in this section we only present standard BP. 
There is a long line of literature providing 
the theoretical and empirical evidences supporting the use BP \cite{Pearl88,YFW03}. 

Under the crowdsourcing graphical model, standard BP operates on two sets of messages: 
the task messages $\{\tx_{i\to a}\}_{(i,a)\in E}$ and 
the worker messages $\{\ty_{a\to i}\}_{(i,a)\in E}$. 
In our iterative algorithm the messages were scalar variables with real values, 
whereas the messages in BP are probability density functions. 
Each task message corresponds to an edge and each worker message also corresponds to an edge. 
The task node $i$ corresponds to random variable $\estimate_i$, 
and the task message from task $i$ to worker $a$, denoted by $\tx_{i\to a}$, 
represents our belief on the random variable $\estimate_i$. 
Then $\tx_{i\to a}$ is a probability distribution over $\{\pm1\}$. 
Similarly, a worker node $a$ corresponds to a random variable $p_a$. 
The worker message $\ty_{a\to i}$ is a probability distribution of $p_a$ over $[0,1]$. 
Following the standard BP framework, 
we iteratively update the messages according to the following rule. 
We start with randomly initialized $\tx_{i\to a}$'s and at $k$-th iteration,  
\begin{eqnarray*}
	\ty^{(k)}_{a\to i}(p_a) &\propto& \crowddist(p_a)\prod_{j\in\partial a \setminus i} \Big\{(p_a+\bp_a+(p_a-\bp_a)A_{ja})\tx_{j\to a}^{(k)}(+1) + 
		( p_a+\bp_a-(p_a-\bp_a)A_{ja})\tx_{j\to a}^{(k)}(-1) \Big\}\;,\\
	\tx^{(k+1)}_{i\to a}(\estimate_i) &\propto& \prod_{b\in\partial i \setminus a} \int \Big( \ty^{(k)}_{b\to i}(p_b)\big(p_b\ind_{(A_{ib}=\estimate_i)}+\bp_b\ind_{(A_{ib}\neq\estimate_i)}\big)\Big)\, dp_b\;,
\end{eqnarray*}
for all $(i,a)\in E$ and for $\bp=1-p$. 
The above update rule only determines the messages up to a scaling, 
where $\propto$ indicates that the left-hand side is proportional to the right-hand side. 
The algorithm produces the same estimates in the end regardless of the scaling. 
After a predefined number of iterations, we make a decision by computing the decision variable 
\begin{eqnarray*}
  \tx_{i}(\estimate_i) &\propto& \prod_{b\in\partial i} \int \Big(\ty^{(k)}_{b\to i}(p_b)\big(p_b\ind_{(A_{ib}=\estimate_i)}+\bp_b\ind_{(A_{ib}\neq\estimate_i)}\big)\Big)\, dp_b\;,
\end{eqnarray*}
and estimating $\estimate_i=\sign\big(\tx_{i}(+)-\tx_{i}(-)\big)$.

In a special case of a Haldane prior, where a worker either always tells the truth or always gives the wrong answer,  
\begin{eqnarray*}
 p_j &=& \left\{ 
      \begin{array}{rl} 0 &\text{ with probability }1/2 \;,\\
      1 &\text{ otherwise } \;, \end{array}
	      \right.
\end{eqnarray*}
the above BP updates boils down to our iterative inference algorithm. 
Let $x_{i\to a}=\log\big(\tx_{i\to a}(+)/\tx_{i\to a}(-)\big)$ denote the log-likelihood of $\tx_{i\to a}(\cdot)$. 
Under Haldane prior, $p_a$ is also a binary random variable. 
We can use $y_{a\to i}=\log\big(\ty_{a\to i}(1)/\ty_{a\to i}(0)\big)$ to denote the log-likelihood of $\ty_{a\to i}(\cdot)$.
After some simplifications, the above BP update boils down to 
\begin{eqnarray*}
	y_{a\to i}^{(k)} &=& \sum_{j\in\partial a \setminus i} A_{ja}x_{j\to a}^{(k-1)} \;,\\
	x_{i\to a}^{(k)} &=& \sum_{b\in\partial i \setminus a} A_{ib}y_{b\to i}^{(k)} \;.
\end{eqnarray*}
This is exactly the same update rule as our iterative inference algorithm 
(cf. Eqs.~\eqref{eq:messageupdate1} ad \eqref{eq:messageupdate2}). 
Thus, our algorithms is belief propagation for a very specific prior. 
Despite this, it is surprising that it performs near optimally 
(with random regular graph for task allocation) for all priors. 
This robustness property is due to the models assumed in this crowdsourcing problem 
and is not to be expected in general. 
%
%
\subsection{Discussion}
\label{sec:discuss}

In this section, we discuss 
several implications of our main results and 
possible future research directions 
in generalizing the model studied in this paper.

\vspace{0.2cm}
\noindent{\bf Below phase transition. } 
We first discuss the performance guarantees in the below threshold regime when $\hl\hr\quality^2<1$. 
As we will show, the bound in \eqref{eq:iterativelimit} always holds even when $\hl\hr\quality^2\leq1$. 
However, numerical experiments suggest that we should stop our algorithm at first iteration when 
we are below the phase transition as discussed in Section~\ref{sec:theory}. We provide 
an upper bound on the resulting error when only one iteration of our iterative inference 
algorithm is used (which is equivalent as majority voting algorithm). 

Notice that the bound in \eqref{eq:iterativelimit} 
is only meaningful when it is less than a half. 
When $\hl\hr\quality^2\leq1$ or $\taskdegree\quality<24 \log 2$, 
the right-hand side of inequality~\eqref{eq:iterativelimit} is always larger than half. 
Hence the upper bound always holds, even without the assumption that $\hl\hr\quality^2>1$, 
and we only have that assumption in the statement of our main theorem to emphasize the phase transition 
in how our algorithm behaves.
 
However, we can also try to get a tighter bound than a trivial half implied from \eqref{eq:iterativelimit} 
in the below threshold regime. 
Specifically, we empirically observe that the error rate increases as the number of iterations $k$ increases. 
Therefore, it makes sense to use $k = 1$. 
In which case, the algorithm essentially boils down to the majority rule. 
We can prove the following error bound which generally holds for any regime of 
$\taskdegree$, $\workerdegree$ and the worker distribution $\cF$.
A proof of this statement is provided in Section~\ref{sec:majority_upperbound}.
\begin{lemma}
	\label{lem:majority_upperbound}
	For any value of $\taskdegree$, $\workerdegree$, and $\ntask$, and any distribution of workers $\cF$, 
	the estimates we get after first step of our algorithm achieve 
	\begin{eqnarray}
		  \label{eq:iterative2}
		  \frac{1}{\ntask}\sum_{i=1}^\ntask\prob\big(\task_i\neq\estimate_i\big) &\leq& e^{-\taskdegree\meanquality^2/4} \;,
	\end{eqnarray}
	where $\meanquality=\E_\cF[2\Bp_j-1]$.
\end{lemma}
Since $\meanquality$ is always between $\quality$ and $\quality^{1/2}$, 
the scaling of the above error exponent is always worse than 
what we have after running our algorithm for a long time (cf. Theorem~\ref{thm:main}). 
This suggests that iterating our inference algorithm helps when $\hl\hr\quality^2>1$ and especially when 
the gap between $\meanquality$ and $\quality$ is large. 
Under these conditions, our approach does significantly better than majority voting (cf. Figure~\ref{fig:comparison}). 
The gain of using our approach is maximized when there exists both good workers and bad workers. 
This is consistent with our intuition that 
when there is a variety of workers, our algorithm can identify the good ones  
and get better estimates. 

\vspace{0.2cm}
\noindent{\bf Golden standard units. } 
Next, consider the variation where we ask questions to 
workers whose answers are already known (also known as `gold standard units'). 
We can use these to assess the quality of the workers.  
There are two ways we can use this information. 
First, we can embed `seed gold units' along with the standard tasks, 
and use these `seed gold units' in turn to perform more informed inference. 
However, we can show that there is no gain in using such `seed gold units'. 
The optimal lower bound of $1/\quality \log (1/\error)$ 
essentially utilizes the existence of oracle that can identify the reliability
of every worker {\em exactly}, i.e. the oracle has a lot more information than what can be gained
by such embedded golden questions. 
Therefore, clearly `seed gold units' {\em do not} help the oracle estimator, 
and hence the order optimality of our approach still holds even if we include all the 
strategies that can utilize these `seed gold units'. 
However, in practice, it is common to use the `seed gold units', 
and this can improve the constant factor in the required budget, but not the scaling.  

Alternatively, we can use `pilot gold units' as  
qualifying or pilot questions that the workers must complete to qualify to participate. 
Typically a taskmaster do not have to pay for these qualifying questions 
and this provides an effective way to increase the quality of the participating workers. 
Our approach can benefit from such `pilot gold units', 
which has the effect of increasing the effective collective quality of the crowd $\quality$. 
Further, if we can `measure' how the distribution of workers change when using 
pilot questions, 
then our main result fully describes how much we can gain by such pilot questions. 
In any case, pilot questions only change the distribution of participating workers, 
and the order-optimality of our approach still holds even if we compare 
all the schemes that use the same pilot questions. 

\vspace{0.2cm}
\noindent{\bf How to optimize over a multiple choices of crowds. } 
We next consider the scenario where 
we have a choice over which crowdsourcing platform to use 
from a set of platforms with different crowds. 
Each crowd might have different worker distributions with different prices. 
Specifically, suppose there are $K$ crowds of workers:
the $k$-th crowd has collective quality $\quality_k$ and 
requires payment of $c_k$ to perform a task. 
Now our optimality result suggests that the per-task cost scales as 
$c_k/q_k \log (1/\error)$ if we only used workers of class $k$. 
More generally,
if we use a mix of these workers, say $\alpha_k$ fraction of workers 
from class $k$, with $\sum_k \alpha_k = 1$, then the effective parameter
$\quality = \sum_k \alpha_k \quality_k$. And subject to this, the optimal per task 
cost scales as $(\sum_k \alpha_k c_k)/(\sum_k \alpha_k \quality_k) \log (1/\error)$. 
This immediately suggests that the optimal choice of fraction $\alpha_k$ must
be such that $\alpha_k > 0$ only if $c_k/\quality_k = \min_{i} c_i/\quality_i$. 
That is, the optimal choice is to select workers only from the classes that have maximal quality per cost 
ratio of $\quality_k/c_k$ over $k\in[K]$. 
One implication of this observation is that it suggests a pricing scheme for crowdsourcing platforms. 
If you are managing a crowdsourcing platform with 
the collective quality $\quality$ and the cost $c$ 
and there is another crowdsourcing platform with $\quality'$ and $c'$, 
you want to choose the cost such that the quality per cost ratio is at least as good as the other crowd:
$\quality/c\,\geq\,\quality'/c'$.

\vspace{0.2cm}
\noindent{\bf General crowdsourcing models. } 
Finally, we consider possible generalizations of our model. 
The model assumed in this paper does not capture 
several factors: tasks with different level of difficulties or 
workers who always answer positive or negative. 
In general, the responses of a worker $j$ to a binary question $i$ may depend on several factors: 
$(i)$ the correct answer to the task;  
$(ii)$ the difficulty of the task; 
$(iii)$ the expertise or the reliability of the worker; 
$(iv)$ the bias of the worker towards positive or negative answers. 
Let $\task_i\in\{+1,-1\}$ represent the correct answer and 
$r_i\in[0,\infty)$ represents the level of difficulty. 
also, let $\alpha_j\in[-\infty,\infty]$ represent the reliability and 
$\beta_j\in(-\infty,\infty)$ represent the bias of worker $j$. 
In formula, a worker $j$'s response to a binary task $i$ can be modeled as 
\begin{eqnarray*}
	\BA_{ij} = \sign(\BZ_{i,j}) \;, 
\end{eqnarray*}
where $\BZ_{i,j}$ is a Gaussian random variable distributed as 
$\BZ_{i,j}\sim\Gauss(\alpha_j \task_i+\beta_j, r_i)$ 
and $\sign(\BZ)=1$ almost surely for $\BZ\sim\Gauss(\infty,1)$. 
A task with $r_i = 0$ is an easy task and large $r_i$ is a difficult task. 
A worker with large positive $\alpha_j$ is more likely to give the right answer 
and large negative $\alpha_j$ is more likely to give the wrong answer. 
When $\alpha_j=0$, the worker gives independent answers regardless of what
the correct answer is. 
A worker with large $\beta_j$ is biased towards positive responses 
and if $\beta_j=0$ then the worker is unbiased. 
A similar model with multi-dimensional latent variables was studied in \cite{Welinder10}.

Most of the models studied in the crowdsourcing literature can be reduced to a special case of this model. 
For example, the early patient-classification model introduced by Dawid and Skene \cite{DS79} 
is equivalent to the above Gaussian model with $r_i=1$. 
Each worker is represented by two latent quality parameters 
$p_j^+$ and $p_j^-$, such that 
\begin{eqnarray*}
  \BA_{ij} &=& \left\{ 
      \begin{array}{rl} \task_i &\text{ with probability }p_j^{\task_i} \;,\\
      -\task_i &\text{ otherwise}\;. \end{array}
	      \right.
\end{eqnarray*}
This model captures the bias of workers.
More recently, Whitehill et al. \cite{whitehill09} introduced another model where 
$\prob(A_{ij}=\task_i|a_i,b_j) = 1/(1+e^{-a_ib_j})$, with worker reliability $a_i$ and task difficulty $b_j$. 
This is again a special case of the above Gaussian model if we set $\beta_j=0$.
The model we study in this paper has  
an underlying assumption that all the tasks share an equal level of difficulty 
and the workers are unbiased.  
It is equivalent to the above Gaussian model with $\beta_j=0$ and $r_i=1$. 
In this case, there is a one-to-one relation between the worker reliability $p_j$ and $\alpha_j$: 
$p_j=Q(\alpha_j)$, where $Q(\cdot)$ is the tail probability of the standard Gaussian distribution. 

%
%
\section{Proof of main results}
\label{sec:proof}

In this section, we provide proofs of the main results.

%
%
\subsection{Proof of the main result in Theorem \ref{thm:main}}
\label{sec:iterativeproof}

By symmetry, we can assume that all $\task_i$'s are $+1$. 
Let $\estimate^{(k)}_i$ denote the resulting estimate of task $i$ 
after $k$ iterations of our iterative inference algorithm defined in Section \ref{sec:algorithm}. 
If we draw a random task $\bI$ uniformly in $[m]$, then we want to compute 
the average error probability, which is the probability that 
we make an error on this randomly chosen task: 
\begin{eqnarray}
	\label{eq:errorbound}
	\frac{1}{m}\sum_{i\in[m]}\prob\big(\,\task_i\neq\estimate^{(k)}_i\,\big) &=& \prob\big(\,\task_\bI\neq\estimate^{(k)}_\bI\,\big)\;.
\end{eqnarray}

We will prove an upper bound on the probability of error in two steps. 
First, we prove that the local neighborhood of a randomly chosen task node $\bI$ is a tree with high probability. 
Then, assuming that the graph is locally tree-like, 
we provide an upper bound on the error using a technique 
known as {\em density evolution}. 

We construct a random bipartite graph $\bG([m]\cup[n],E)$ according to the configuration model.  
We start with $[\ntask]\times[\taskdegree]$ half-edges for task nodes and 
$[\nworker]\times[\workerdegree]$ half-edges for the worker nodes, 
and pair all the $\ntask\taskdegree$ task half-edges to 
the same number of worker half-edges 
according to a random permutation of $[\ntask\taskdegree]$. 

Let $\bG_{i,k}$ denote a subgraph of $\bG([m]\cup[n],E)$ 
that includes all the nodes whose distance from the `root' $i$ is at most $k$. 
At first iteration of our inference algorithm, 
to estimate the task $i$, 
we only use the responses provided 
by the workers who were assigned to task $i$. 
Hence we are performing inference on the local neighborhood $\bG_{i,1}$. 
Similarly, when we run $k$ iterations of our (message-passing) inference algorithm to estimate a task $i$, 
we only run inference on local subgraph $\bG_{i,2k-1}$. 
Since we update both task and worker messages, 
we need to grow the subgraph by distance two at each iteration. 
When this local subgraph is a tree, 
then we can apply density evolution to analyze the probability of error. 
When this local subgraph is not a tree, we can make a pessimistic assumption that  
an error has been made to get an upper bound on the actual error probability.  
\begin{eqnarray}
	\label{eq:bounddecompose}
	 \prob\big(\,\task_\bI\neq\estimate^{(k)}_\bI\,\big) &\leq& \prob\big(\,\bG_{\bI,2k-1}\text{ is not a tree}\,\big) + \prob\big(\,\bG_{\bI,2k-1} \text{ is a tree}\text{ and } \task_\bI\neq\estimate^{(k)}_\bI \,\big)\;.
\end{eqnarray}
Next lemma bounds the first term and shows that the probability that a local subgraph is not a tree 
vanishes as $\ntask$ grows.  A proof of this lemma is provided in Section~\ref{sec:prooflocaltree}. 
\begin{lemma}
	\label{lem:localtree}
	For a random $(\taskdegree,\workerdegree)$-regular bipartite graph generated according to the configuration model, 
	\begin{eqnarray}
		\label{eq:boundlocalnotree}
		\prob\big(\, \bG_{\bI,2k-1} \text{\rm{ is not a tree }} \,\big) &\leq& 
		\big((\taskdegree-1)(\workerdegree-1)\big)^{2k-2}\frac{3\taskdegree\workerdegree}{\ntask}\;. 
	\end{eqnarray}
\end{lemma}
Then, to bound the second term of \eqref{eq:bounddecompose}, we provide a sharp upper bound on the error probability conditioned on that $\bG_{\bI,2k-1}$ is a tree. 
Let $x_i^{(k)}$ denote the decision variable for task $i$ after $k$ iterations 
of the iterative algorithm such that $\estimate^{(k)}_i = \sign(x_i^{(k)})$. 
Then, we make an error whenever this 
decision variable is negative. When this is exactly zero, we make a random decision, in which case we make 
an error with probability half. Then, 
\begin{eqnarray}
	\label{eq:boundde}
	\prob\big(\, \task_\bI\neq\estimate^{(k)}_\bI \,\big|\, G_{\bI,k} \text{ is a tree}\,\big) &\leq& \prob\big(\,x_\bI^{(k)} \leq 0\,\big|\, G_{\bI,k} \text{ is a tree}\,\big) \;.
	\end{eqnarray} 
To analyze the distribution of the decision variable on a locally tree-like graph, 
we use a standard probabilistic analysis technique known as `density evolution' 
in coding theory or `recursive distributional equations' 
in probabilistic combinatorics \cite{RU08,MM09}. 
Precisely, we use the following equality that  
\begin{eqnarray}
	\label{eq:mainbound}
	\prob\big(\,x_\bI^{(k)}\leq 0\,\big|\, G_{\bI,k} \text{ is a tree}\,\big) &=& \prob\big(\,\hBx^{(k)}\leq0\,\big) \;,
\end{eqnarray}
where $\hBx^{(k)}$ is defined through density evolution equations 
\eqref{eq:de1}, \eqref{eq:de2} and \eqref{eq:dedecision} in the following. 
We will prove in the following that when $\hl\hr\quality^2>1$, 
\begin{eqnarray}
	\label{eq:boundsubgaussian} 
	\prob\big(\,\hBx^{(k)}\leq0\,\big) &\leq& e^{-\taskdegree\quality/(2\sigma_k^2)}\;.
\end{eqnarray}
Together with equations \eqref{eq:mainbound}, \eqref{eq:boundde}, \eqref{eq:boundlocalnotree}, 
\eqref{eq:bounddecompose}, and \eqref{eq:errorbound}, 
this finishes the proof of Theorem \ref{thm:main}.

\vspace{0.1cm}
\noindent{\bf Density evolution.}
At iteration $k$ the algorithm operates on a set of messages 
$\{x_{i\to j}^{(k)}\}_{(i,j)\in E}$ and $\{y_{j\to i}^{(k)}\}_{(i,j)\in E}$. 
If we chose an edge $(i,j)$ uniformly at random, 
the values of $x$ and $y$ messages on that randomly chosen edge define random variables 
whose randomness comes from random choice of the edge, 
any randomness introduced by the inference algorithm, 
the graph, and the realizations of $\Bp_j$'s and $\BA_{ij}$'s. 
Let $\Bx^{(k)}$ denote this random variable corresponding to the message $x_{i\to j}^{(k)}$ 
and $\By_{p}^{(k)}$ denote the random variable corresponding to $y_{j\to i}^{(k)}$ conditioned on 
the latent worker quality being $p$ for randomly chosen edge $(i,j)$.

As proved in Lemma~\ref{lem:localtree}, 
the $(\taskdegree,\workerdegree)$-regular random graph locally converges in distribution to 
a $(\taskdegree,\workerdegree)$-regular tree with high probability. 
On a tree, there is a recursive way of defining the distribution of messages 
$\Bx^{(k)}$ and $\By_{p}^{(k)}$. 
At initialization, we initialize the worker messages with Gaussian random variable with 
mean one and variance one. The corresponding random variable 
$\By_{p}^{(0)}\sim\cN(1,1)$, which at initial step is independent of 
the worker quality $p$, fully describes the distribution of $y_{j\to i}^{(0)}$ 
for all $(i,j)$. 
At first iteration, the task messages are updated according to 
$x_{i\to j}^{(1)}=\sum_{j'\in\partial i \setminus j} \BA_{ij'} y_{j'\to i}^{(0)}$. 
If we know the distribution of $\BA_{ij'}$'s and $y_{j'\to i}^{(0)}$'s, 
we can update the distribution of $x_{i\to j}^{(1)}$. 
Since we are assuming a tree, all $x_{i\to j}^{(1)}$ are independent. 
Further, because of the symmetry in the way we construct our random graph, 
all $x_{i\to j}^{(1)}$'s are identically distributed. 
Precisely, they are distributed according to $\Bx^{(1)}$ defined in \eqref{eq:de1}. 
This recursively defines $\Bx^{(k)}$ and $\By^{(k)}$ 
through the {\em density evolution equations} 
in \eqref{eq:de1} and \eqref{eq:de2} \cite{MM09}. 

Let us first introduce a few definitions first. 
Here and after, we drop the superscript $k$ 
denoting the iteration number whenever it is clear from the context. 
Let $\Bx_{b}$'s and $\By_{p,a}$'s be independent 
random variables distributed according to $\Bx$ and $\By_{p}$ respectively. 
Also, $\Bz_{p,a}$'s and $\Bz_{p,b}$'s are 
independent random variables distributed according to $\Bz_{p}$, where 
\begin{eqnarray*}
\Bz_{p}&=& \left\{ 
   \begin{array}{rl} +1 &\text{ with probability }p \;,\\
   - 1 &\text{ with probability } 1-p \;.\end{array}\right. 
\end{eqnarray*}
This represents the answer given by a worker conditioned on the worker having quality parameter $p$.
Let $\Bp\sim\cF$ be a random variable distributed according to 
the distribution of the worker's quality $\cF$ over $[0,1]$. 
Then $\Bp_a$'s are independent random variable distributed according to $\Bp$. 
Further, $\Bz_{p,b}$'s and $\Bx_{b}$'s are independent, and  
$\Bz_{\Bp_a,a}$'s and $\By_{\Bp_a,a}$'s 
are conditionally independent conditioned on $\Bp_a$. 

We initialize $\By_p$ with a Gaussian distribution, 
whence it is independent of the latent variable $p$: 
$\By^{(0)}_p \sim \Gauss(1,1)$. 
Let $\stackrel{d}{=}$ denote equality in distribution.
Then, for $k\in\{1,2,\ldots\}$, 
the task messages are distributed as the sum 
of $\taskdegree-1$ incoming messages that are independent and 
identically distributed according to $\By_\Bp^{(k-1)}$ and weighted by i.i.d. responses: 
\begin{eqnarray}
 \label{eq:de1}
 \Bx^{(k)} &\stackrel{d}{=}& \sum_{a\in[\taskdegree-1]} \Bz_{\Bp_a,a} \By^{(k-1)}_{\Bp_a,a} \;. 
\end{eqnarray}
Similarly, the worker messages (conditioned on the latent worker quality $p$) are distributed as the sum 
of $\workerdegree-1$ incoming messages that are independent and 
identically distributed according to  $\Bx^{(k)}$ and weighted by i.i.d. responses: 
\begin{eqnarray}
 \label{eq:de2}
 \By^{(k)}_{p}   &\stackrel{d}{=}& \sum_{b\in[\workerdegree-1]} \Bz_{p,b} \Bx^{(k)}_b \;. 
\end{eqnarray}
For the decision variable $x_\bI^{(k)}$ on a randomly chosen task $\bI$, we have 
\begin{eqnarray}
 \label{eq:dedecision}
 \hBx^{(k)} &\stackrel{d}{=}& \sum_{i\in[\taskdegree]} \Bz_{\Bp_i,i}\By^{(k-1)}_{\Bp_i,i} \;. 
\end{eqnarray}
 
Numerically or analytically computing the densities in \eqref{eq:de1} and \eqref{eq:de2} exactly is not 
computationally feasible 
when the messages take continuous values as is the case for our algorithm. 
Typically, heuristics are used to approximate the densities 
such as quantizing the messages, approximating the density with simple functions, 
or using Monte Carlo method to sample from the density. 
A novel contribution of our analysis is that 
we prove that the messages are sub-Gaussian 
using recursion, and we provide an upper bound on the parameters in a closed form. 
This allows us to prove the sharp result on the error bound that decays exponentially. 

\vspace{0.2cm}
{\bf Mean and variance computation.} 
To give an intuition on how the messages behave, 
we describe the evolution of the mean and the variance of 
the random variables in \eqref{eq:de1} and \eqref{eq:de2}. 
Let $\Bp$ be a random variable distributed according to the measure $\cF$.
Define $m^{(k)}\equiv\E[\Bx^{(k)}]$, $\hm^{(k)}_\Bp\equiv\E[\By^{(k)}_\Bp|\Bp]$, 
$v^{(k)}\equiv{\rm Var}(\Bx^{(k)})$, and $\hv^{(k)}_\Bp\equiv{\rm Var}(\By^{(k)}_\Bp|\Bp)$.
Also let $\hl=\taskdegree-1$ and $\hr=\workerdegree-1$ to simplify notation.
Then, from \eqref{eq:de1} and \eqref{eq:de2} we get that  
\begin{eqnarray*}
  m^{(k)}   &=&  \hl\,\E_\Bp\big[(2\Bp-1)\hm_\Bp^{(k-1)}\big] \;,\\
  \hm^{(k)}_p &=&  \hr\,(2p-1)m^{(k)}\;,\\
  v^{(k)}   &=&  \hl \left\{\E_\Bp[ \hv_\Bp^{(k-1)} + \big(\hm_\Bp^{(k-1)}\big)^2 ] - \Big(\E_\Bp\big[(2\Bp-1)\hm_\Bp^{(k-1)}\big]\Big)^2\right\} \;,\\
  \hv^{(k)}_p &=&  \hr \left\{v^{(k)}+(m^{(k)})^2 - \big((2p-1)m^{(k)}\big)^2 \right\}\;.
\end{eqnarray*}
Recall that $\meanquality=\E[2\Bp-1]$ and $q=\E[(2\Bp-1)^2]$.
Substituting $\hm_\Bp$ and $\hv_\Bp$ we get the following evolution of 
the first and the second moment of the random variable $x^{(k)}$.
\begin{eqnarray*}
  m^{(k+1)}   &=&  \hl\hr q m^{(k)}  \;,\\
  v^{(k+1)}   &=&  \hl\hr v^{(k)} + \hl\hr(m^{(k)})^2 (1-q)(1+\hr q) \;.
\end{eqnarray*}
Since $\hm_\Bp^{(0)}=1$ and $\hv^{(0)}=1$ as per our assumption, 
we have $m^{(1)}=\meanquality\hl$ and $v^{(1)}=\hl(4-\meanquality^2)$.
This implies that $m^{(k)}=\meanquality\hl(\hl\hr q)^{k-1}$, and 
$v^{(k)} = av^{(k-1)}+bc^{k-2}$, 
with $a=\hl\hr$, $b=\meanquality^2\hl^3\hr(1-q)(1+\hr q)$, and $c=(\hl\hr q)^2$.
After some algebra, it follows that $v^{(k)} = v^{(1)}a^{k-1} + b c^{k-2} \sum_{\ell=0}^{k-2} (a/c)^\ell$. 

For $\hl\hr q^2>1$, we have $a/c<1$ and 
\begin{eqnarray*}
  v^{(k)} = \hl(4-\meanquality^2)(\hl\hr)^{k-1} + (1-q)(1+\hr q)\meanquality^2\hl^2(\hl\hr q)^{2k-2} \frac{1-1/(\hl\hr q^2)^{k-1}}{\hl\hr q^2-1} \;.
\end{eqnarray*}

The first and second moment of the decision variable $\hBx^{(k)}$ in 
\eqref{eq:dedecision} can be computed using a similar analysis: 
$\E[\hBx^{(k)}] = (\taskdegree/\hl)m^{(k)}$ and ${\rm Var}(\hBx^{(k)})= (\taskdegree/\hl)v^{(k)}$.
In particular, we have 
\begin{eqnarray*}
  \frac{{\rm Var}(\hBx^{(k)})}{\E[\hBx^{(k)}]^2} = 
    \frac{\hl(4-\meanquality^2)}{\taskdegree\hl\meanquality^2(\hl\hr q^2)^{k-1}} + 
    \frac{\hl(1-q)(1+\hr q)}{\taskdegree(\hl\hr q^2-1)} \Big(1-\frac{1}{(\hl\hr q^2)^{k-1}}\Big)\;.
\end{eqnarray*}
Applying Chebyshev's inequality, it immediately follows that $\prob(\hBx^{(k)}<0)$ is 
bounded by the right-hand side of the above equality. 
This bound is weak compared to the bound in Theorem \ref{thm:main}.  
In the following, we prove a stronger result using the sub-Gaussianity of $\Bx^{(k)}$.
But first, let us analyze what this weaker bound gives for different regimes 
of $\taskdegree$, $\workerdegree$, and $\quality$, 
which indicates that the messages exhibit a fundamentally different behavior 
in the regimes separated by a phase transition at $\hl\hr\quality^2=1$. 

In a `good' regime where we have $\hl\hr\quality^2 > 1$, 
the bound converges to a finite limit as the number of iterations $k$ grows. Namely, 
\begin{eqnarray*}
  \lim_{k\rightarrow \infty} \prob(\hx^{(k)}<0) &\leq& \frac{\hl(1-q)(1+\hr q)}{\taskdegree(\hl\hr q^2-1)}\;.
\end{eqnarray*}
Notice that the upper bound converges to $(1-\quality)/(\taskdegree\quality)$ as $\hl\hr\quality^2$ grows. 
This scales in the same way as the known bounds for using the left singular vector directly 
for inference (cf. \cite{KOS11allerton}). 
In the case when $\hl\hr q^2<1$, the same analysis gives
\begin{eqnarray*}
  \frac{{\rm Var}(\hx^{(k)})}{\E[\hx^{(k)}]^2} &=& e^{\Theta(k)}\;.
\end{eqnarray*}
Finally, when $\hl\hr q^2=1$, we get 
$v^{(k)} = (\hl\hr)^k + \hl\hr(1-q)(1+\hr q)(\hl\hr q)^{2k-2} k$,
which implies 
\begin{eqnarray*}
  \frac{{\rm Var}(\hx^{(k)})}{\E[\hx^{(k)}]^2} &=& \Theta(k)\;.
\end{eqnarray*}

\vspace{0.1cm}
\noindent{\bf Analyzing the density.}
Our strategy to provide a tight upper bound on $\prob(\hBx^{(k)}\leq0)$
is to show that $\hBx^{(k)}$ 
is sub-Gaussian with appropriate parameters  
and use the Chernoff bound.
A random variable $\Bz$ with mean $m$ is said to be {\em sub-Gaussian} with parameter $\tsigma$ if 
for all $\lambda\in\reals$ the following inequality holds: 
\begin{eqnarray*}
	\E[e^{\lambda\Bz}] &\leq& e^{m\lambda+(1/2)\tsigma^2\lambda^2}\;.
\end{eqnarray*}
Define 
\begin{eqnarray*}
	\tsigma_k^2 &\equiv& 2\hl(\hl\hr)^{k-1} + \meanquality^2\hl^3\hr(3\quality\hr+1)(\quality\hl\hr)^{2k-4}\frac{{1-(1/\quality^2\hl\hr)^{k-1}}}{{1-(1/\quality^2\hl\hr)}} \;, 
\end{eqnarray*}
and $m_k \equiv \meanquality\hl(\quality\hl\hr)^{k-1}$ for $k\in\Z$. 
We will first show that, $\Bx^{(k)}$ is sub-Gaussian with mean $m_k$ and parameter $\tsigma_k^2$ 
for a regime of $\lambda$ we are interested in. 
Precisely, we will show that for $|\lambda|\leq1/(2m_{k-1}\hr)$, 
\begin{eqnarray}
    \label{eq:recursionform}
   \E[e^{\lambda\Bx^{(k)}}] \;\leq\;  e^{m_k\lambda+(1/2)\tsigma_k^2\lambda^2} \;. 
\end{eqnarray}
By definition, due to distributional independence, 
we have $\E[e^{\lambda \hBx^{(k)}}] = \E[e^{\lambda \Bx^{(k)}}]^{(\taskdegree/\hl)}$. 
Therefore, it follows from \eqref{eq:recursionform} that $\hBx^{(k)}$ satisfies 
$\E[e^{\lambda\hBx^{(k)}}] \leq  e^{(\taskdegree/\hl)m_k\lambda+(\taskdegree/2\hl)\tsigma_k^2\lambda^2}$. 
Applying the Chernoff bound with $\lambda=-m_k/(\tsigma_k^2)$, we get 
\begin{eqnarray}
	\label{eq:subgaussianbound}
	\prob\big( \hBx^{(k)} \leq 0 \big) \;\leq\; \E\big[e^{\lambda\hBx^{(k)}}\big] \;\leq\; e^{-\taskdegree\, m_k^2/(2\,\hl\,\tsigma_k^2)} \;,
\end{eqnarray}
Since $m_km_{k-1}/(\tsigma_k^2) \leq \meanquality^2\hl^2(q\hl\hr)^{2k-3}/(3\meanquality^2q\hl^3\hr^2(q\hl\hr)^{2k-4}) = 1/(3\hr)$, 
it is easy to check that $|\lambda|\leq1/(2m_{k-1}\hr)$.
This implies the desired bound in \eqref{eq:boundsubgaussian}.  

Now we are left to prove that $\Bx^{(k)}$ is sub-Gaussian with appropriate parameters. 
We can write down a recursive formula for the evolution of 
the moment generating functions of $\Bx$ and $\By_p$ as
\begin{eqnarray}
 \label{eq:me1}
 \E\big[e^{\lambda \Bx^{(k)}}\big]   &=& \Big( \E_\Bp\Big[ \Bp \E[e^{\lambda \By^{(k-1)}_\Bp}|\Bp] +\bBp \E[e^{-\lambda \By^{(k-1)}_\Bp}|\Bp] \Big] \Big)^{\hl} \;, \\
 \label{eq:me2}
 \E\big[e^{\lambda \By^{(k)}_p}\big] &=& \Big(p\E\big[e^{\lambda \Bx^{(k)}}\big] + \bp \E\big[e^{-\lambda \Bx^{(k)}}\big]  \Big)^{\hr} \;,
\end{eqnarray}
where $\bp=1-p$ and $\bBp=1-\Bp$. 
We can prove that these are sub-Gaussian using induction. 

First, for $k=1$, we show that 
$\Bx^{(1)}$ is sub-Gaussian with 
mean $m_1=\meanquality\hl$ and parameter $\tsigma_1^2=2\hl$, 
where $\meanquality\equiv\E[2\Bp-1]$. 
Since $\By_p$ is initialized as Gaussian with unit mean and variance, 
we have $\E[e^{\lambda \By^{(0)}_p}] = e^{\lambda+(1/2)\lambda^2}$ regardless of $p$. 
Substituting this into \eqref{eq:me1}, we get for any $\lambda$, 
\begin{eqnarray} 
	\label{eq:firstiteration}
	\E\Big[e^{\lambda \Bx^{(1)}}\Big] \;=\; \Big(\E[\Bp] e^{\lambda} + (1-\E[\Bp]) e^{-\lambda}\Big)^\hl e^{(1/2)\hl\lambda^2} 
		\;\leq\; e^{\hl\meanquality\lambda + \hl\lambda^2 }\;,
\end{eqnarray}
where the inequality follows from the fact that 
$ae^z+(1-a)e^{-z}\leq e^{(2a-1)z+(1/2)z^2}$ for any $z\in\reals$ and $a\in[0,1]$ (cf. \cite[Lemma A.1.5]{AS08}).

Next, assuming $\E[e^{\lambda\Bx^{(k)}}] \leq e^{m_k\lambda+(1/2)\tsigma_k^2\lambda^2}$ 
for $|\lambda|\leq 1/(2m_{k-1}\hr)$, 
we show that $\E[e^{\lambda\Bx^{(k+1)}}] \leq e^{m_{k+1}\lambda+(1/2)\tsigma_{k+1}^2\lambda^2}$ 
for $|\lambda|\leq 1/(2m_{k}\hr)$, and compute appropriate $m_{k+1}$ and $\tsigma_{k+1}^2$.  
Substituting the bound $\E[e^{\lambda\Bx^{(k)}}] \leq e^{m_k\lambda+(1/2)\tsigma_k^2\lambda^2}$  
in \eqref{eq:me2}, we get 
$ \E[e^{\lambda \By_p^{(k)}}] \leq (p e^{m_k\lambda}+\bp e^{-m_k\lambda})^\hr e^{(1/2)\hr\tsigma_k^2\lambda^2}$.
Further applying this bound in \eqref{eq:me1}, we get 
\begin{eqnarray}
	\label{eq:recursion}
	\E\Big[e^{\lambda \Bx^{(k+1)}}\Big] &\leq& \left(\E_\Bp\Big[ \Bp( \Bp e^{m_k\lambda}+\bBp e^{-m_k\lambda})^\hr + 
	\bBp(\Bp e^{-m_k\lambda}+\bBp e^{m_k\lambda})^\hr \Big]\right)^\hl e^{(1/2)\hl\hr\tsigma_k^2\lambda^2}\;.  
\end{eqnarray}
To bound the first term in the right-hand side, we use the next key lemma. 
A proof of this lemma is provided in Section~\ref{sec:proofmomentlemma}. 
\begin{lemma}
	\label{lem:moment}
	For any $|z|\leq1/(2\hr)$ and $\Bp\in[0,1]$ such that $q=\E[(2\Bp-1)^2]$, we have 
	\begin{eqnarray*}
	  \E_\Bp\Big[\Bp(\Bp e^z+\bBp e^{-z})^\hr+\bBp(\bBp e^z+\Bp e^{-z})^\hr\Big] &\leq& e^{\quality\hr z + (1/2)(3\quality\hr^2+\hr)z^2}\;.
	\end{eqnarray*}
\end{lemma}
Applying this inequality to \eqref{eq:recursion} gives 
\begin{eqnarray*}
	\E[e^{\lambda \Bx^{(k+1)}}] 
	&\leq&  e^{ \quality\hl\hr m_k \lambda  + (1/2) \big( (3\quality\hl\hr^2+\hl\hr)m_k^2 + \hl\hr\tsigma_k^2\big)\lambda^2 } \;,
\end{eqnarray*}
for $|\lambda|\leq1/(2m_k\hr)$. 
In the regime where $q\hl\hr\geq1$ as per our assumption, $m_k$ is non-decreasing in $k$. 
At iteration $k$, the above recursion holds for $|\lambda|\leq \min\{1/(2m_1\hr), \ldots, 1/(2m_{k-1}\hr)\}=1/(2m_{k-1}\hr)$. 
Hence, we get the following recursion for $m_k$ and $\tsigma_k$ such that 
\eqref{eq:recursionform} holds for $|\lambda|\leq 1/(2m_{k-1}\hr)$.
\begin{eqnarray*}
	m_{k+1} &=& \quality\hl\hr m_k \;,\\
	\tsigma_{k+1}^2 &=& (3\quality\hl\hr^2+\hl\hr)m_k^2+\hl\hr\tsigma_k^2 \;.  
\end{eqnarray*}
With the initialization $m_1=\meanquality\hl $ and $\tsigma_1^2=2\hl$, 
we have $m_k=\meanquality\hl(\quality\hl\hr)^{k-1}$ for $k\in\{1,2,\ldots\}$ and 
$\tsigma_k^2 = a\tsigma_{k-1}^2 + bc^{k-2}$ for $k\in\{2,3,\ldots\}$, 
with $a=\hl\hr$, $b=\meanquality^2\hl^2(3\quality\hl\hr^2+\hl\hr)$, and $c=(\quality\hl\hr)^2$.
After some algebra, it follows that 
$\tsigma_k^2 = \tsigma_1^2a^{k-1} + b c^{k-2} \sum_{\ell=0}^{k-2} (a/c)^\ell$. 
For $\hl\hr\quality^2\neq1$, we have $a/c\neq1$, 
whence $\tsigma_k^2 = \tsigma_1^2a^{k-1} + b c^{k-2} (1-(a/c)^{k-1})/(1-a/c)$. 
This finishes the proof of \eqref{eq:recursionform}.

%
%
\subsection{Proof of Lemma \ref{lem:localtree}}
\label{sec:prooflocaltree} 
Consider the following discrete time random process that generates the random graph 
$\bG_{\bI,2k-1}$ starting from the root $\bI$. 
At first step, we connect $\taskdegree$ worker nodes to 
node $\bI$ according to the configuration model, 
where $\taskdegree$ half-edges are matches to 
a randomly chosen subset of $\nworker\workerdegree$ worker half-edges of size $\taskdegree$. 
Let $\alpha_1$ denote the probability that the resulting graph is a tree, 
that is no pair of edges are connected to the same worker node. 
Since there are ${\taskdegree \choose 2}$ pairs and each pair of half-edges are 
connected to the same worker node with probability $(\workerdegree-1)/(\nworker\workerdegree-1)$: 
\begin{eqnarray*}
	\alpha_1 &\leq& {\taskdegree \choose 2} \frac{\workerdegree-1}{\nworker\workerdegree-1} \;.
\end{eqnarray*} 
Similarly, define 
\begin{eqnarray*}
	\alpha_t &\equiv& \prob(\bG_{\bI,2t-1}\text{ is not a tree}\,|\,\bG_{\bI,2t-2}\text{ is a tree}\,) \text{ , and }\\  
	\beta_t &\equiv& \prob(\bG_{\bI,2t-2}\text{ is not a tree}\,|\,\bG_{\bI,2t-3}\text{ is a tree}\,) \;.
\end{eqnarray*}
Then, 
\begin{eqnarray}
	\label{eq:treetele}
	\prob(\bG_{\bI,2k-1}\text{ is not a tree }) &\leq& \alpha_1 + \sum_{t=2}^{k} \big(\alpha_t + \beta_t\big)\;. 
\end{eqnarray}
We can upper bound $\alpha_t$'s and $\beta_t$'s in a similar way. 
For generating $\bG_{\bI,2t-1}$ conditioned on $\bG_{\bI,2t-2}$ being a tree, 
there are $\taskdegree(\hl\hr)^{t-1}$ half-edges, where $\hl=\taskdegree-1$ and $\hr=\workerdegree-1$. 
Among ${\taskdegree(\hl\hr)^{t-1} \choose 2}$ pairs of these half-edges, 
each pair will be connected to the same worker
 with probability at most $(\workerdegree-1)/(\workerdegree(\nworker-\sum_{a=1}^{t-1}\taskdegree(\hl\hr)^{a-1})-1)$, where $\sum_{a=1}^{t-1}\taskdegree(\hl\hr)^{a-1}$ is the total number of worker nodes that are assigned so far in $\bG_{\bI,2t-2}$. 
Then, 
\begin{eqnarray*}
	\alpha_t &\leq& \frac{\taskdegree^2(\hl\hr)^{2t-2}}{2} \frac{\workerdegree-1}{\workerdegree(\nworker-(\taskdegree((\hl\hr)^{t-2}-1))/(\hl\hr-1))-1}\\
	&\leq& \frac{\taskdegree^2(\hl\hr)^{2t-2}}{2(\nworker-\taskdegree(\hl\hr)^{t-2}/2)}\\
	&\leq& \frac{\taskdegree^2(\hl\hr)^{2t-2}}{\nworker} + \frac{\taskdegree(\hl\hr)^{t-2}}{\nworker}\\
	&\leq& \frac{3\taskdegree^2(\hl\hr)^{2t-2}}{2\nworker}\;,
\end{eqnarray*}
where the second inequality follows from the fact that $(a-1)/(b-1)\leq a/b$ for all $a\leq b$ 
and $\hl\hr\geq2$ as per our assumption, 
and in the third inequality we used the fact that $\alpha_t$ is upper bounded by one and 
the fact that for $f(x)=b/(x-a)$ which is upper bounded by one, we have $f(x) \leq (2b/x) + (2a/x)$. 
Similarly, we can show that 
\begin{eqnarray*}
	\beta_t & \leq & \frac{3\taskdegree^2(\hl\hr)^{2t-2}}{\hl^2\ntask}\;.
\end{eqnarray*}
Substituting $\alpha_t$ and $\beta_t$ into \eqref{eq:treetele}, we get that 
\begin{eqnarray*}
	\prob(\bG_{\bI,2k-1}\text{ is not a tree }) &\leq& (\hl\hr)^{2k-2}\frac{3\taskdegree\workerdegree}{\ntask}\;.
\end{eqnarray*}

%
%
\subsection{Proof of Lemma \ref{lem:moment}}
\label{sec:proofmomentlemma}

By the fact that 
$ae^b+(1-a)e^{-b}\leq e^{(2a-1)b+(1/2)b^2}$ for any $b\in\reals$ and $a\in[0,1]$, 
we have $\Bp e^z+\bBp e^{-z} \leq e^{ (2\Bp-1)z + (1/2)z^2 }$ almost surely.
Applying this inequality once again, we get 
\begin{eqnarray*}
	\E\Big[\Bp(\Bp e^z+\bBp e^{-z})^\hr+\bBp(\bBp e^z+\Bp e^{-z})^\hr\Big] &\leq& \E\Big[e^{(2\Bp-1)^2\hr z+ (1/2)(2\Bp-1)^2\hr^2 z^2}\Big] e^{(1/2)\hr z^2}. 
\end{eqnarray*}
Using the fact that $e^{a} \leq 1+a+0.63a^2$ for $|a|\leq5/8$, 
\begin{align*}
	& \E\Big[e^{(2\Bp-1)^2\hr z+ (1/2)(2\Bp-1)^2\hr^2 z^2}\Big] \\ 
	&\;\;\;\;\; \leq \;\; \E\Big[ 1+(2\Bp-1)^2\hr z+ (1/2)(2\Bp-1)^2\hr^2 z^2 + 0.63((2\Bp-1)^2\hr z+ (1/2)(2\Bp-1)^2\hr^2 z^2)^2 \Big] \\
	&\;\;\;\;\; \leq \;\; 1+q\hr z+ (3/2)q\hr^2 z^2 \\
	&\;\;\;\;\; \leq \;\; e^{q\hr z+ (3/2)q\hr^2 z^2} \;,
\end{align*}
for $|z|\leq1/(2\hr)$. 
This proves Lemma \ref{lem:moment}.

%
%
\subsection{Proof of a bound on majority voting in Lemma \ref{lem:majority}}
\label{sec:proof_majority}

Majority voting simply follows what the majority of workers agree on. 
In formula, $\estimate_i=\sign(\sum_{j\in W_i}{A_{ij}})$, 
where $W_i$ denotes the neighborhood of node $i$ in the graph. 
It makes a random choice when there is a tie. 
We want to compute a lower bound on $\prob( \estimate_i \neq \task_i )$. 
Let $x_i=\sum_{j\in W_i}A_{ij}$. 
Assuming $\task_i=+1$ without loss of generality, the error rate is lower bounded by $\prob(x_i<0)$. 
After rescaling, $(1/2)(x_i+\taskdegree)$ is a standard binomial random variable ${\rm Binom}(\taskdegree,\alpha)$, 
where $\taskdegree$ is the number of neighbors of the node $i$, 
$\alpha=\E[\Bp_j]$, and  
by assumption each $A_{ij}$ is one with probability $\alpha$. 

It follows that $\prob( x_i = -l+2k ) = ((\taskdegree!)/(\taskdegree-k)!k!)\alpha^k (1-\alpha)^{l-k}$. 
Further, for $k\leq \alpha \taskdegree-1$, the probability distribution function is 
monotonically increasing. 
Precisely, 
$$\frac{\prob( x_i = -\taskdegree+2(k+1) )}{\prob( x_i = -\taskdegree+2k )} \geq \frac{\alpha(\taskdegree-k)}{(1-\alpha)(k+1)} \geq \frac{\alpha (\taskdegree-\alpha \taskdegree +1)}{(1-\alpha)\alpha \taskdegree } > 1\;,$$
where we used the fact that the above ratio is decreasing in $k$ whence 
the minimum is achieved at $k=\alpha \taskdegree -1$ under our assumption. 

Let us assume that $\taskdegree$ is even, so that $x_i$ take even values. 
When $\taskdegree$ is odd, the same analysis works, but $x_i$ takes odd values.
Our strategy is to use a simple bound: 
$\prob( x_i < 0 ) \geq k \prob( x_i = -2k  )$. 
By assumption that $\alpha=\E[\Bp_j]\geq 1/2$, 
For an appropriate choice of $k=\sqrt{l}$, 
the right-hand side closely approximates the error probability.  
By definition of $x_i$, it follows that 
\begin{eqnarray}
  \label{eq:lbproof}
  \prob\left( x_i = -2\sqrt{\taskdegree}\right) &=& {\taskdegree\choose \taskdegree/2+\sqrt{\taskdegree} } \alpha^{\taskdegree/2-\sqrt{l}} \big(1-\alpha\big)^{\taskdegree/2+\sqrt{\taskdegree}} \;.
\end{eqnarray}

Applying Stirling's approximation, we can show that 
\begin{eqnarray}
	\label{eq:stirling1}
	{\taskdegree\choose \taskdegree/2+\sqrt{\taskdegree} } &\geq& \frac{C_2}{\sqrt{\taskdegree}}\,2^l\;,
\end{eqnarray}
for some positive constant $C_2$. 
We are interested in the case where worker quality is low, that is $\alpha$ is close to $1/2$. 
Accordingly, for the second and third terms in \eqref{eq:lbproof}, we expand in terms of $2\alpha-1$. 
\begin{align}
  \label{eq:stirling2}
 &\log\left( \alpha^{\taskdegree/2-\sqrt{\taskdegree}} \Big(1-\alpha\Big)^{\taskdegree/2+\sqrt{\taskdegree}}\right) \nonumber\\
 &\;\;\;\;\;\; = \Big(\frac{\taskdegree}{2}-\sqrt{\taskdegree}\Big)\Big(\log(1+(2\alpha-1))-\log(2)\Big) + \Big(\frac{\taskdegree}{2}+\sqrt{\taskdegree}\Big)\Big(\log(1-(2\alpha-1))-\log(2)\Big) \nonumber\\ 
 &\;\;\;\;\;\; = -\taskdegree\log(2) -\frac{\taskdegree(2\alpha-1)^2}{2} +O(\sqrt{\taskdegree(2\alpha-1)^4}) \;.
\end{align}

We can substitute \eqref{eq:stirling1} and \eqref{eq:stirling2} in 
\eqref{eq:lbproof} to get the following bound: 
\begin{eqnarray} 
  \prob( x_i < 0 ) &\geq& \exp\big\{-C_3 (\taskdegree(2\alpha-1)^2+1)\big\} \;,
\end{eqnarray}
for some positive constant $C_3$. 

Now, let $\taskdegree_i$ denote the degree of task node $i$, such that $\sum_i \taskdegree_i=\taskdegree\ntask$. 
Then for any $\{\task_i\}\in\{\pm1\}^\ntask$, any distribution of $\Bp$ such that  $\meanquality=\E[2\Bp-1]=2\alpha-1$, 
and any non-adaptive task assignment for $\ntask$ tasks, the following lower bound is true. 
\begin{eqnarray*} 
	\frac1\ntask\sum_{i\in[\ntask]}\prob\big(\task_i\neq\estimate_i\big) &\geq& \frac{1}{m} \sum_{i=1}^{m}e^{-C_3(\taskdegree_i\meanquality^2+1)} \\
	&\geq& e^{-C_3(\taskdegree\meanquality^2+1)} \;, 
\end{eqnarray*}
where the last inequality follows from convexity of the exponential function.
Under the spammer-hammer model, where $\meanquality=\quality$ this gives 
\begin{eqnarray*}
		\min_{\tau\in\cT_\taskdegree}\;\; \max_{\task\in\{\pm1\}^\ntask,\crowddist\in\crowddist_\quality}\;\; \frac1\ntask\sum_{i\in[\ntask]}\prob\big(\task_i\neq\estimate_i\big) &\geq& e^{-C_3(\taskdegree\quality^2+1)} \;. 
\end{eqnarray*}
This finishes the proof of lemma.

%
%
\subsection{Proof of a bound on the adaptive  schemes in Theorem \ref{thm:minimax_adaptive}}
\label{sec:proof_minimax_adaptive}

In this section, we prove that, even with the help of an oracle, 
the probability of error cannot decay faster than $e^{-C\taskdegree\quality}$. 
We consider an labeling algorithm which has access to 
an oracle that knows the reliability of every worker (all the $p_j$'s). 
At $k$-th step, after the algorithm assign $T_k$ and all the $|T_k|$ answers are collected from the $k$-th worker, 
the oracle provides the algorithm with $p_k$. 
Using all the previously collected answers $\{A_{ij}\}_{j\leq k}$ and 
the worker reliability $\{p_j\}_{j\leq k}$, the algorithm makes a decision on the next task assignment $T_{k+1}$. 
This process is repeated until a stopping criterion is met, 
and the algorithm outputs the optimal estimate of the true labels. 
The algorithm can compute the maximum likelihood estimates,  
which is known to minimize the probability of making an error. 
Let $W_i$ be the set of workers assigned to task $i$, then 
\begin{eqnarray}
	\estimate_i = \sign\Big(\sum_{j\in W_i} \log\Big(\frac{p_j}{1-p_j}\Big)A_{ij} \Big)\;. \label{eq:ml}
\end{eqnarray}

We are going to show that there exists a family of distributions $\crowddist$ such that 
for any stopping rule and any task assignment scheme, the probability of error is lower bounded by $e^{-C\taskdegree\quality}$. 
We define the following family of distributions according to the spammer-hammer model with imperfect hammers. 
We assume that $\quality\leq a^2$ and 
\begin{eqnarray*}
  p_j &=& \left\{ 
      \begin{array}{rl} 1/2 &\text{ with probability }1-(\quality/a^2) \;,\\
      1/2(1+a) &\text{ with probability } \quality/a^2 \;, \end{array}
	      \right.
\end{eqnarray*}
such that $\E[(2p_j-1)^2]=\quality$.

Let $W_i$ denote the set of workers assigned to task $i$ when the algorithm has stopped.
Then $|W_i|$ is a random variable representing 
the total number of workers assigned to task $i$. 
The oracle estimator knows all the values necessary to compute the error probability of each task.   
Let $\conditionalerror_i=\E[\ind(\task_i\neq\estimate_i)|\{A_{ij}\},\{p_j\}]$ be the random variable representing the error probability as computed by the oracle estimator, 
conditioned on the $|W_i|$ responses we get from the workers and their reliability $p_j$'s. 
We are interested in identifying how the average budget $(1/m)\sum_i\E\big[|W_i|\big]$ depends on the 
achieve average error rate $(1/m)\sum_i\E[\conditionalerror_i]$. 
In the following we will show that for any task $i$, 
independent of which task allocation scheme is used, it is necessary that  
\begin{eqnarray}
	\label{eq:conditionalminimax}
	\E\big[|W_i|\big] &\geq& \frac{0.27}{\quality}\log\Big(\frac{1}{2\E[\conditionalerror_i]}\Big) \;.
\end{eqnarray}
By convexity of $\log(1/x)$ and Jensen's inequality, this implies that 
\begin{eqnarray*}
	\frac{1}{m}\sum_{i=1}^{m}\E\big[|W_i|\big] &\geq& \frac{0.27}{\quality}\log\Big(\frac{1}{2(1/m)\sum_{i=1}^{m}\E[\conditionalerror_i]}\Big) \;.
\end{eqnarray*}
Since the total number of queries has to be consistent, we have $\sum_j|T_j|=\sum_i|W_i|\leq\ntask\taskdegree$. 
Also, by definition $\E[\conditionalerror_i]=\prob(\task_i\neq\estimate_i)$. 
Then, from the above inequality, we get 
$(1/\ntask)\sum_{i\in[\ntask]}\prob(\task_i\neq\estimate_i) \geq (1/2)e^{-(1/0.27)\quality\taskdegree}$, 
which finishes the proof of the theorem. 
Note that this bound holds for any value of $m$. 

Now, we are left to prove that the inequality \eqref{eq:conditionalminimax} holds. 
Focusing on a single task $i$, 
since we know who the spammers are and spammers give us no information about the task, 
we only need the responses from the reliable workers in order to make an optimal estimate as per \eqref{eq:ml}. 
The conditional error probability $\conditionalerror_i$ of the optimal estimate depends on the realizations of 
the answers $\{A_{ij}\}_{j\in W_i}$ and the worker reliability $\{p_j\}_{j\in W_i}$. 
The following lower bound on the error only depends on the number of reliable workers, 
which we denote by $\taskdegree_i$. 

Without loss of generality, let $\task_i=+1$. Then, if all the reliable workers provide `$-$' answers, 
the maximum likelihood estimation would be `$-$' for this task. 
This leads to an error. Therefore, 
\begin{eqnarray*}
	\conditionalerror_i &\geq& \prob(\text{all $\taskdegree_i$ reliable workers answered $-$}) \\
	&=& \frac{1}{2}\Big(\frac{1-a}{2}\Big)^{\taskdegree_i}\;,
\end{eqnarray*}
for all the realizations of $\{A_{ij}\}$ and $\{p_j\}$. 
The scaling by half ensures that the above inequality holds even when $\taskdegree_i=0$. 
By convexity and Jensen's inequality, it follows that 
\begin{eqnarray*}
	\E\big[\taskdegree_i\big] &\geq& \frac{\log\big(2\E[\conditionalerror_i]\big)}{\log\big((1-a)/2\big)}\;. 
\end{eqnarray*}
When we recruit $|W_i|$ workers, we see $\taskdegree_i=(\quality/a^2)|W_i|$ reliable ones on average. 
Formally, we have $\E[\taskdegree_i]=(\quality/a^2)\E[|W_i|]$. 
Again applying Jensen's inequality, we get 
\begin{eqnarray*}
	\E\big[|W_i|\big] &\geq& \frac{1}{\quality}\frac{a^2}{\log\big((1-a)/2\big)}\log\big(2\E[\conditionalerror_i]\big)\;.
\end{eqnarray*}
Maximizing over all choices of $a\in(0,1)$, we get 
\begin{eqnarray*}
	\E\big[|W_i|\big] &\geq& -\log\big(2\E[\conditionalerror_i]\big)\frac{0.27}{\quality}   \;,
\end{eqnarray*}
which in particular is true with $a = 0.8$. 
For this choice of $a$, the result holds in the regime where $q\leq0.64$. 
Notice that by changing the constant in the bound, 
we can ensure that the result holds for any values of $\quality$. 
This finishes the proof of \eqref{eq:conditionalminimax}.

%
%
\subsection{Proof of a bound with one iteration in Lemma \ref{lem:majority_upperbound}}
\label{sec:majority_upperbound}
The probability of making an error after one iteration of our algorithm 
for node $i$ is $\prob(\task_i\neq\estimate_i^{(1)})\leq\prob(\hBx_i\leq0)$, 
where $\hBx_i=\sum_{j\in\partial i} \BA_{ij} \By_{j\to i}^{(1)}$. 
Assuming $t_i=+$, without loss of generality, $\BA_{ij}$ is $+1$ with probability $\E[\Bp]$ 
and $-1$ otherwise. All $\By_{j\to i}^{(1)}$'s are initialized as 
Gaussian random variables with mean one and variance one. 
All these random variables are independent of one another at this initial step. 
Hence, the resulting random variable $\hBx_i$ 
is a sum of a shifted binomial random variable $2({\rm Binom}(\taskdegree,\E[\Bp])-\taskdegree)$ 
and a zero-mean Gaussian random variable $\cN(0,\taskdegree)$. 
From calculations similar to \eqref{eq:firstiteration}, it follows that 
\begin{eqnarray*}
      \E\Big[e^{\lambda\hBx^{(1)}}\Big] &\leq& e^{\taskdegree\meanquality\lambda + \taskdegree\lambda^2 } \;\leq\; e^{-(1/4)\taskdegree\meanquality^2}\;,
\end{eqnarray*}
where we choose $\lambda=-\meanquality/2$.
By Chernoff's inequality, this implies the lemma for any value of $\ntask$.
%
%

%
%

\section{Conclusion}
\label{sec:conclusion}

We conclude with some limitations of our results and interesting research directions. 

{\em 1. More general models.} In this paper, we provided an order-optimal task assignment scheme  
and an order-optimal inference algorithm for that task assignment  
assuming a probabilistic crowdsourcing model. 
In this model, we assumed that each worker makes a mistake randomly 
according to a worker specific quality parameter. 
Two main simplifications we make here is that,  
first, the worker's reliability does not depend on 
whether the task is a positive task or a negative task, 
and second, all the tasks are equally easy or difficult. 
The main remaining challenges in developing inference algorithms for crowdsourcing is 
how to develop a solution for more generic models formally described in Section~\ref{sec:discuss}.
When workers exhibit bias and can have heterogeneous quality parameters 
depending on the correct answer to the task, 
spectral methods using low-rank matrix approximations nicely generalize 
to give an algorithmic solution. 
Also, it would be interesting to find algorithmic solutions with performance guarantees 
for the generic model where tasks difficulties are taken into account. 

{\em 2. Improving the constant.} We prove our approach is minimax optimal up to a constant factor. 
However, there might be another algorithm with better constant factor than our inference algorithm. 
Some modification of the expectation maximization or the belief propagation 
might achieve a better constant 
compared to our inference algorithm. 
It is an interesting research direction to find such an algorithm and give 
an upper bound on the error probability that is smaller than what we have in our main theorem. 

{\em 3. Instance-optimality.} The optimality of our approach is proved under the worst-case worker distribution. 
However, it is not known whether our approach is instance-optimal or not under the non-adaptive scenario. 
It would be important to prove lower bounds for all worker distributions or 
to find a counter example where another algorithm achieves 
a strictly better performance for a particular worker distribution 
in terms of the scaling of the required budget. 

{\em 4. Phase transition.} 
We empirically observe that there is a phase transition around $\hl\hr\quality^2=1$. 
Below this, no algorithm can do better than majority voting. 
This phase transition seems to be an algorithm-independent and fundamental property of the problem 
(and the random graph). 
It might be possible to formally prove the fundamental difference in the way information propagates 
under the crowdsourcing model. 
Such phase transition has been studied for a simpler model of broadcasting on trees 
in information theory and statistical mechanics  \cite{EKPS00}. 

%
%
\bibliographystyle{amsalpha}

\bibliography{mturk}

\end{document}